\documentclass[11pt,a4paper]{article}
\usepackage[margin=1in]{geometry}
\usepackage[T1]{fontenc}
\usepackage{lmodern}
\usepackage{microtype}
\usepackage[authoryear,round]{natbib}
\usepackage{authblk}
\usepackage[hidelinks]{hyperref}

\usepackage{graphicx}%
\usepackage{multirow}%
\usepackage{amsmath,amssymb,amsfonts}%
\usepackage{amsthm}%
\usepackage{mathrsfs}%
\usepackage[title]{appendix}%
\usepackage{xcolor}%
\usepackage{textcomp}%
\usepackage{manyfoot}%
\usepackage{booktabs}%
\usepackage{algorithm}%
\usepackage{algorithmicx}%
\usepackage{algpseudocode}%
\makeatletter
\renewcommand{\theHALG@line}{\thealgorithm.\arabic{ALG@line}}
\makeatother
\usepackage{listings}%
\usepackage[utf8]{inputenc}
\usepackage{alphabeta}         %
\usepackage{bm}                %
\usepackage{subcaption}        %
\usepackage[section]{placeins} %
\usepackage{longtable}         %
\usepackage{xstring}           %
\usepackage{anyfontsize}       %
\usepackage[acronym,nopostdot,nonumberlist,nogroupskip,hyperfirst=false]{glossaries}
\usepackage[nameinlink,capitalize]{cleveref} %

\theoremstyle{plain}%
\newtheorem{theorem}{Theorem}%
\newtheorem{lemma}{Lemma}%
\newtheorem{corollary}{Corollary}

\theoremstyle{remark}%
\newtheorem{remark}{Remark}%

\theoremstyle{definition}%
\newtheorem{definition}{Definition}%

\usepackage{array}
\newcolumntype{L}[1]{>{\raggedright\let\newline\\\arraybackslash\hspace{0pt}}m{#1}}
\newcolumntype{C}[1]{>{\centering\let\newline\\\arraybackslash\hspace{0pt}}m{#1}}
\newcolumntype{R}[1]{>{\raggedleft\let\newline\\\arraybackslash\hspace{0pt}}m{#1}}

\graphicspath{{figures/}}

\DeclareMathOperator*{\argmin}{arg\,min}

\let\svthefootnote\thefootnote
\newcommand\nnfootnote[1]{%
  \let\thefootnote\relax\footnotetext{#1}%
  \let\thefootnote\svthefootnote%
}

\newcommand{\sigp}[3][\mathbf{x}]{{
 \IfStrEq{#1}{!}{
 		x^{#3}_{#2}
 	}{
 		{#1}^{#3}_{#2}
 	}
}}

\newcommand{\sigw}[2]{{
\IfStrEq{#2}{}{
		\mathbf{w}^{#2}_{#1}
	}{
		w^{#2}_{#1} %
	}
}}

\newcommand{\up}[1]{\mathord{\mathrm{#1}}}

\newenvironment{calign}
    {
    \begin{equation}
    \begin{aligned}
    }
    {
    \end{aligned}
    \end{equation}
    }

\newenvironment{cgather}
    {
    \begin{equation}
    \begin{gathered}
    }
    {
    \end{gathered}
    \end{equation}
    }

\newcommand{\mathglos}[4]{
  \newglossaryentry{#1}{name={\ensuremath{#3}}, description={\textbf{#2}: #4.}, sort={#2},}
}

\makenoidxglossaries

\mathglos{dirac}{Dirac}{δ_x}{a member of $P(\mathcal{X})$ with weight only at $x \in \mathcal{X}$, where $δ_x (X) = 1$ if $ x \in X \subset \mathcal{X} $ and $0$ otherwise}

\mathglos{indices}{Indices}{i,j}{usually corresponding to the $i^\text{th}$ basis pattern $b^i$ and $j^\text{th}$ combined pattern $c^j$ respectively}

\mathglos{radon_sp}{Radon space}{ \mathcal{M} (\mathcal{X}) }{loosely, all $m$ such that $ m(X) \leq \infty $ for all $X \subset \mathcal{X}$}

\mathglos{radon_cone}{Radon cone}{ \mathcal{M}_+ (\mathcal{X}) }{those $m \in \mathcal{M} (\mathcal{X})$ with $ m(X) \geq 0 $ for all $X \subset \mathcal{X}$}

\mathglos{prob_measures}{Probability measures}{P (\mathcal{X}) }{i.e. those $m \in \mathcal{M}_+ (\mathcal{X})$ with $m(\mathcal{X}) = 1$}

\mathglos{non_neg_cone}{Non-negative cone}{\mathbb{R}_+^d}{i.e. $\{ \mathbf{x} \in \mathbb{R}^{d} : x_i \geq 0 \}$}

\mathglos{simplex}{Simplex}{S^{d-1}}{i.e. $\{ \mathbf{x} \in \mathbb{R}^{d}_+ : \Vert \mathbf{x} \Vert_1 = 1 \}$}

\mathglos{pushforward}{Pushforward}{T_\# μ}{defined as $ (T_\# μ) (Y) = μ( T^{-1} (Y) ) $ (for $ Y \subset \mathcal{Y} $) for measure $μ \in \mathcal{M} (\mathcal{X})$ and map $T : \mathcal{X} \rightarrow \mathcal{Y}$}

\mathglos{couplings}{Couplings}{\mathcal{U} (μ, ν)}{a set containing all $ { π \in P (\mathcal{X} \times \mathcal{Y}) } $  that have $μ$ and $ν$ as marginals}

\mathglos{tuple}{Tuple}{( u^1, \ldots, u^N ) \in U^N}{tuple or ordered set with elements $u \in U$}

\mathglos{tensor}{Tensor coordinates}{ (\mathbf{v}^1 \ldots \mathbf{v}^N ) \in V^N }{expressed as concatenated slices of lower order tensor coordinates}

\mathglos{pnorm}{P-norm}{\Vert \mathbf{x} \Vert_p^p}{(to the $\text{p}^\text{th}$ power) or $ \sum_{i=1}^d \vert x_i \vert ^p $ where $ \mathbf{x} \in \mathbb{R}^d $}

\mathglos{range}{Integer range}{[ N ]}{i.e. the set $ \{ 1, 2, \ldots, N-1, N \} \subset \mathbb{N} $}

\mathglos{shift_down}{Shift-down operator}{\mathbf{u}^{\downarrow s}}{with indices $ ( \mathbf{u}^{\downarrow s} )_i  = u_{i-s} $}

\mathglos{disc_conv}{Discrete cross-correlation}{*_S}{where $ \mathbf{u} *_S \mathbf{v} = \sum_{s=1}^{S} v_s \mathbf{u}^{\downarrow s-1} $ }

\mathglos{abundances}{Abundances}{a^{j}_i, \mathbf{A}, \mathbf{a}^j, \mathbf{a}_i}{matrix entries, where we take a full slice over omitted indices}

\mathglos{bases}{Bases}{b^i, \up{B}}{single pattern and tuple of respectively}
\mathglos{bases_warped}{Basis variations}{\widetilde{b}^{ij}, \widetilde{\up{B}}^j}{corresponding to combined pattern $j$, single and tuple of respectively}

\mathglos{comb_pat}{Combined/mixed patterns}{c^j, \up{C}}{single and tuple of respectively}

\mathglos{gsd}{Ground space distance}{d_\mathcal{X}}{distance between points in $\mathcal{X}$}

\mathglos{psd}{Pattern space distance}{d_\mathcal{P}}{distance between points in $\mathcal{P}$}

\mathglos{dist_mat}{Distance matrix}{\mathbf{D} \in \mathbb{R}^{N \times N}}{usually between the supports of a given pair of finite measures}

\mathglos{comb_func}{Pattern combination function}{κ(\mathbf{a}; \up{B})}{mechanism for combining multiple patterns into a resultant mixed pattern}

\mathglos{idx_lim}{Index limits}{K,M,N}{size of the basis, combined patterns set and measure support (maximum) respectively}

\mathglos{pat_sp}{Pattern space}{\mathcal{P}}{data space in which both pure and mixed examples live}

\mathglos{tp}{Transport plan}{π \in P(\mathcal{X} \times \mathcal{Y})}{for one given pair of measures}

\mathglos{tp_disc}{Transport plan weights}{\mathbf{P} \in \mathbb{R}^{N \times N} }{for a discrete plan}

\mathglos{wass_dist}{Wasserstein $p$ distance}{ W^{d_\mathcal{X}}_p }{over measures with support in $(\mathcal{X}, d_\mathcal{X})$}

\mathglos{wass_dist_short}{Wasserstein $p$ distance}{ W_p }{shorthand taking $d_{\mathcal{\mathbb{R}}^N} (x,y) = \Vert x - y \Vert_p$}

\mathglos{sig_wei}{Signature weights}{W_μ}{set of weights corresponding to each of the Diracs that compose a finite measure $μ$}

\mathglos{sig_pos}{Signature positions}{X_μ}{set of positions of the Diracs that compose a finite measure $μ$}

\mathglos{gss}{Ground spaces}{ \mathcal{X}, \mathcal{Y} }{underlying spaces over which measures may be defined on}

\newacronym{aadpa}{AADPA}{Average Absolute Difference in Percentage Abundances}
\newacronym{afd}{AFD}{Agile Factor Decomposition}
\newacronym{bccw}{BCCW}{Basis-Component Centered Wasserstein}
\newacronym{ed}{ED}{Euclidean Distance}
\newacronym{emd}{EMD}{Earth Movers Distance}
\newacronym{hi}{HI}{Hyperspectral Imaging}
\newacronym{gs}{GS}{Ground Space}
\newacronym{gd}{GD}{Ground Distance}
\newacronym{gc}{GC}{Ground Cost}
\newacronym{gmm}{GMM}{Gaussian Mixture Model}
\newacronym{ibcw}{IBCW}{Inter Basis Centered Wasserstein}
\newacronym{iccw}{ICCW}{Inter Components Centered Wasserstein}
\newacronym{mle}{MLE}{Maximum Likelihood Estimate}
\newacronym{nmf}{NMF}{Non-negative Matrix Factorisation}
\newacronym{pca}{PCA}{Principal Component Analysis}
\newacronym{onmf}{oNMF}{Overlapping NMF}
\newacronym{nwca}{NWCA}{North-west Corner Algorithm}
\newacronym{nwcr}{NWCR}{North-west Corner Rule}
\newacronym{ot}{OT}{Optimal Transport}
\newacronym{rmse}{RMSE}{Root Mean Square Error}
\newacronym{pxrd}{PXRD}{Powder X-ray Diffraction}
\newacronym{ip}{IP}{Inverse Problem}

\glsaddall

\makenoidxglossaries

\begin{document}

\title{Robust blind unmixing: A geometric approach to overcoming basis variation}
\author[1]{Dumitru Mirauta}
\author[1]{Vladimir V. Gusev\thanks{Corresponding author: \href{mailto:vladimir.gusev@liverpool.ac.uk}{vladimir.gusev@liverpool.ac.uk}.}}
\author[2]{Michael W. Gaultois}
\author[2]{Matthew J. Rosseinsky}
\author[1]{Yannis Goulermas\thanks{This author has passed away.}}
\affil[1]{School of Computer Science and Informatics, University of Liverpool, Ashton Street, Liverpool, L69 3BX, Merseyside, United Kingdom}
\affil[2]{Materials Innovation Factory, University of Liverpool, Oxford Street, Liverpool, L7 3NY, Merseyside, United Kingdom}
\renewcommand{\Authfont}{\normalsize}
\renewcommand{\Affilfont}{\small}
\date{}
\hypersetup{pdftitle={Robust blind unmixing: A geometric approach to overcoming basis variation}, pdfauthor={Dumitru Mirauta, Vladimir V. Gusev, Michael W. Gaultois, Matthew J. Rosseinsky, Yannis Goulermas}}
\maketitle

\begin{abstract}
Signal separation problems are common in science.
A prominent example of this occurs during the use of diffraction or spectroscopy to identify the individual components of a mixture by measuring it.
In the simplest case,
the measured signal is a linear combination of basis patterns corresponding to the constituent parts.
The unmixing problem is to infer all or some of these basis patterns and abundances of components from measurements of distinct mixtures.
One of the core challenges of this task is the variation of the basis from mixture to mixture due to noise and the exact physics of the measurement process.
This is usually addressed with tailored model-based and parametric methods that are then limited in use to specific application domains by the nature of the assumptions made.
We propose a novel geometric approach to unmixing problems which views the generation of data during measurement through a metric space lens, thereby shifting the focus from parametrised models to a general relationship between basis transformations and the corresponding geometry.
We take advantage of the optimal transport distances to capture commonly occurring basis variations, and use minimisation of in-class variance of candidate solutions to drive the optimisation.
We pay special attention to the one-dimensional case due to its practical importance and availability of efficient distance and transport map routines.
The effectiveness of our approach is demonstrated on a range of unmixing tasks using random Gaussian mixture models, simulated powder X-ray diffraction, and laboratory hyperspectral imaging datasets.
\end{abstract}


\section{Introduction} \label{sec1}

The objects and processes of investigation in many scientific disciplines require advanced characterisation, often with multiple instrumental techniques to examine different aspects.
In many cases, measurements generate a nonnegative response across some domain (e.g., across a range of angles, frequencies, energies, etc.). Examples of such responses include \gls{pxrd} patterns, \gls{hi} pixels, and Raman spectra.
For techniques producing real valued but periodic traces (e.g. infra-red spectroscopy, ion cyclotron resonance mass spectrometry, or even of audio waves) frequency distributions may be considered.
These response profiles can be useful as a pseudoidentifier of the underlying object, where a similar profile is likely to refer to a similar---though not necessarily the same---object.
A response may be characterised in itself, by a ``pattern'' of highs and lows,
large and small features in different locations.
As such, in this work, these response profiles will generally be referred to as patterns.
For identification purposes, 
the magnitude of features relative to one another is more instructive than the absolute magnitude of any individual feature.
As such, we will treat patterns as normalised distributions.

Crucially, the investigation of what may be considered a single underlying object, may produce a range of patterns.
For instance, the \gls{pxrd} pattern of a ``pure phase'' is composed of a series of peaks;
these peaks can individually change position or shape, due to small variation in crystal structure induced as external parameters---such as temperature or pressure---change.
Astronomical \gls{hi} pixels may exhibit Doppler shift as relative speed to an object changes, 
or Earth \gls{hi} may exhibit transformations due to atmospheric scattering, subsurface scattering, and illumination.
In audio, the same structure (e.g. chord or overtone series of a voice) can change pitch and will have an associated (log) frequency shift.
As such, multiple distinct points in pattern space can represent the same basis pattern, we call this basis variation.

Further, it is not always feasible to probe an object of interest alone. 
For instance, a single \gls{pxrd} sample may contain responses from many ``pure phases''.
In such cases, the available response profiles will be a combination of responses from multiple objects.
In order for the pattern(s) of interest to be isolated, it/they must first be separated.
The set of all distinct patterns that can play a role in mixing will be referred to as the basis.
In linear mixing, there is no basis variation, and mixtures are produced through linear combination of the basis.

Though linear combination is a common feature in many types of mixed data, across disciplines, so is basis variation.
Here, it is variations of the basis patterns that are combined.
With this, a set of patterns corresponding to the same set of objects, combined in the same proportions, may produce two different mixtures.
As such, in this work a mixing model is defined by providing both a basis variation mechanism and a combination mechanism.
This flexible umbrella model covers data encountered in materials discovery \cite{ermon2014combiFD, RelaxedProjection, Cherni20}, Remote sensing \cite{Borsoi20, Heylen14, Ma14, Aharoni22},  Biology \cite{Hosseini15, Yan19, Glatz11}, Astrophysics \cite{Boulais21, Rameau21, Nakhostin16, liu18}, Food testing \cite{Holland98}, Music transcription \cite{Smaragdis, Flamary16} and other forms of audio separation (e.g. cocktail party problem). %
Image segmentation in computer vision \cite{ho2017multilevel}, multiple curve clustering/alignment \cite{Gaffney04} and distribution template estimation \cite{Boissard11} are also related.

While mixing produces a set of combined patterns from a basis, blind unmixing aims to recover a basis solely from a given set of combined patterns.
While there is of course information loss in forming a single mixture, 
we may be able to patch together a good estimate of a basis from many mixed examples. 
Here domain-wise (cross-position) variation works in our favor, as different parts of a basis may be visible or obstructed in different mixtures depending on how the bases have transformed.
Blind unmixing can of course be seen as an \gls{ip} \cite{tarantola05, devito05}, where the basis is a model parameter which essentially fixes a mixing model, which otherwise depends on per mixture variables.
Otherwise, nonblind unmixing seeks to perform separation of the constituents in each mixture, according to a given basis.
With sufficiently complex basis variation, there may be many ways of generating a given mixed example from a given basis, making even a nonblind ``inversion'' rather difficult.

Consider the importance of tackling the aforementioned basis variation in materials discovery; %
as sufficiently flexible unmixing methods are not available, much expert time can be directed towards manual analysis, instead of being able to focus on interpretation of higher level data and further experiment planning.
The impact from the reduction of such a bottleneck may be felt far, as new materials enable new technologies and techniques for many other fields,
including towards the improvement of the very hardware that could be helping accelerate the materials discovery process.
A mixing method that is not too specific to any single kind of characterisation technique is particularly beneficial here, as many kinds of characterisation techniques are in use.

We classify previous solutions to this problem as either (parametric) model fitting based, or geometric.
Fitting based methods of course require the specification of a model to be fitted, however, it may be difficult to capture the multitude of transformation mechanisms required across the application domains, in a single, reasonably constrained model.
A geometric method instead views mixing and its inverse problem geometrically. 
For instance, the linear unmixing case translates to a problem of recognising vertices in a point cloud that is expected to span a simplex.
As many families of parametric mixing model can produce geometries with the same quality, they could all be covered by one geometric approach targeting this shared quality.

The more general form of the algorithm we suggest operates in an arbitrarily chosen metric space $ (\mathcal{P}, d_\mathcal{P}) $.
If the spectra of interest exhibit transformations that are Lipschitz in this space, 
then this space has a key property.
Namely, that up to some transformation parameter limit (which depends on the distinctiveness of the true basis), we show that the basis variations form clean, separable clusters (see \cref{separability_sec}).
This expected regularity can be exploited when searching for a way to pull apart our mixtures such that the resultant components correspond to plausible basis variants.
The set of patterns (or representations thereof) $\mathcal{P}$ is a choice as well, for instance 1D peak patterns might be expressed through a space of vectors, of \gls{gmm} parameters, of abstract measures, etc.
This in turn may affect the kind of distance that can be defined between patterns.

After this our focus shifts to the choice of Wasserstein space, and generating a practical algorithm that can execute in time scales feasible for interactive laboratory exploration.
Wasserstein space can represent distribution-like data that is continuous or discrete, with weights on uniform or non-uniform grids, when grids are shared between discrete distributions or when not.
Wasserstein distances can compare any combination of the above, and deal with transformations that would not be continuous if we were required to continue respecting a shared uniform grid for instance, such as when shifting data provided on a uniform grid.
Wasserstein distances are Lipschitz to shifting, broadening and more \cite{Engquist16}.
Overall, our approach to unmixing is novel in:
\begin{enumerate}
  \item Allowing for non-Euclidean perturbation rather than explicitly fitting a fixed basis variation model.
  \item Focusing on how a given mixture can be broken apart rather than on how a similar mixture might be put together.
  \item Using the geometry of mixing to inform plausible decompositions, rather than a regressive cost.
  \item Allowing for sparse and dynamic descriptors through signatures, rather than fixed grids, which can more accurately represent Wasserstein centroids and enable efficient separation.
\end{enumerate}

This paper is structured as follows.
In \cref{prob_def_sec} we present our unmbrella model for mixing problems, establishing what is shared among them, and what can arbitrarily change.
\cref{OT_sec} gives an overview of Wasserstein distances and their place within \gls{ot} distances.
In \cref{existing_works_sec} we discuss existing fitting based and geometric approaches to unmixing.
\cref{main_contribution} contains our main contributions, our general approach, followed by the Wasserstein specialisation, and a practical implementation for 1D.
A fast implementation hinges on efficient calculation of both Wasserstein distances, and crucially, transport plans, as such \cref{fast_wasserstein_sec} discusses efficient computaton of these.
In \cref{results_sec} we compare the effectiveness of our method vs \gls{nmf} and \gls{onmf}, when dealing with data with basis variation in three contexts; \gls{gmm}, \gls{pxrd} and Hyperspectral data.
Finally in \cref{conclusion} we will conclude and discuss how the approach could be further extended.

\section{Preliminaries} \label{preliminaries_sec}

Lower indices will correspond to vector indices while upper indices may sometimes enumerate non-vector quantities, such as discrete measures.
By $ X \subseteq \mathcal{X} $ we will implicitly mean measurable subsets.

\glsfindwidesttoplevelname
\setglossarystyle{alttree}
\printnoidxglossary[sort=letter,title={List of Symbols}]
\printnoidxglossary[type=\acronymtype,title=Acronyms]

\subsection{Formal definition of mixing and unmixing} \label{prob_def_sec}

An unmixing problem seeks to invert the outputs of a mixing process for the corresponding inputs, as such, let us first define mixing.
Let $ \mathcal{P} $ be the space of allowed (pure or mixed) patterns, such as $ \mathcal{P} = \mathbb{R}^N $ as in many cases, or the space of distributions over the real line $ \mathcal{P} = P (\mathbb{R}) $ as we will later consider.

The \textbf{basis} $\up{B} = ( b^1, \ldots , b^K) \in \mathcal{P}^K $ is a tuple of such patterns, representing prototypes to be drawn from in the mixing process.

A \textbf{combination function} $ { κ : \mathbb{R}^K \rightarrow \mathcal{P} } $ will describe the mechanism which prototypes may be combined through, such that each mixture is a function of its mixing proportions (here we take the basis to be fixed).
\begin{figure}[htbp]
  \centering
  \includegraphics[width=\linewidth]{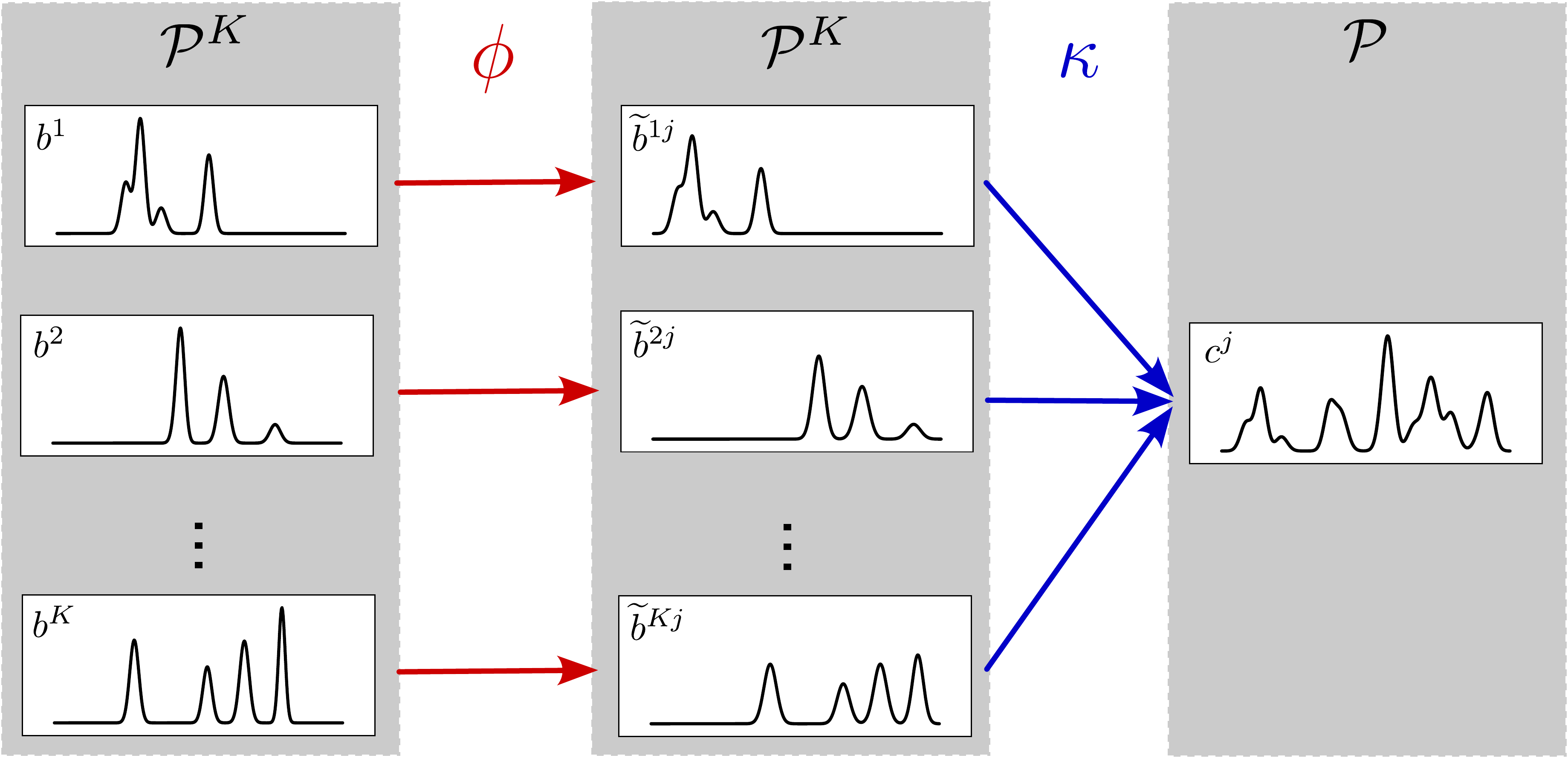}
  \caption{
  A mixing model operating in the space of pattern $\mathcal{P}$, incorporating independent variation of basis patterns prior to combination.
  Left: Basis patterns $b^i$. 
  Middle: The varied basis copies $\widetilde{b}^{ij}$ that correspond to the mixture $c^j$. 
  Right: The resulting mixture $c^j$. 
Transitions from left to middle and from middle to right are provided through the variation $ϕ$ and combination $κ$ functions respectively (illustrated in red and blue respectively).} \label{basis_var_fig}
\end{figure}
The simplest and most common example is linear superposition $ κ^\text{lin} (\mathbf{a} ; \up{B}) = \sum_{i=1}^K a_i b^i $.
In the case of $\mathcal{P}$ being a space of distributions, mixing fractions will need to belong to a simplex
$S^{d-1} = \{ \mathbf{x} \in \mathbb{R}^{d}_+ : \Vert \mathbf{x} \Vert_1 = 1 \}$, for valid combinations, and so $ { κ : S^{K-1} \rightarrow \mathcal{P} } $.
Other examples include bilinear mixing $ κ (\mathbf{A} ; \up{B} ) = \sum_{m,n} a_{mn} {b}^{m} \odot {b}^{n} $ in some hyperspectral unmixing methods (where $\mathcal{P} = \mathbb{R}^N$ and $\odot$ is the Hadamard product) \cite{Heylen14} or weighted Wasserstein centroids as in \cite{Schmitz18}. %

While in the linear case a mixture might be constructed as $c^j = κ^\text{lin} (\mathbf{a}^j ; \widetilde{\up{B}}^j ) $,
in the presence of \textbf{independent variation of each basis pattern}, we could interpret each mixture to be constructed according to its own, subtly different basis  $c^j = κ^\text{lin} (\mathbf{a}^j ; \widetilde{\up{B}}^j ) $, 
with $ \widetilde{\up{B}}^j = (\widetilde{b}^{1j}, \ldots, \widetilde{b}^{Kj}) $.
The \textbf{basis variation model}, if parameterised, can be captured by a function $ ϕ : Λ \rightarrow \mathcal{P} $ in the variation parameters $\bm{λ} \in Λ$.
For example, letting $\mathcal{P}$ be continuous functions over the real line, we could have shifting and broadening $ \widetilde{b} = ϕ (\bm{λ} ; b) = b ( λ_2 \cdot ( x - λ_1 ) ) $.
For brevity, let this also act elementwise over a tuple of patterns, such that we may write $\widetilde{\up{B}}^j = ϕ(\bm{Λ}^j ; \up{B} )$ where $\bm{Λ}^j$ collects the individual parameters of each basis for mixture $j$.

In full, we say that a mixture/combined pattern $c^j$ arises according to the \textbf{mixing model} $ ψ (\mathbf{a}^j, \bm{Λ}^j ; \up{B}) = κ (\mathbf{a}^j ; ϕ ( \bm{Λ}^j ; \up{B}) ) $, generating the dataset ${ \up{C} = \{c^j : j \in [M] \} }$.
The \textbf{blind unmixing problem} seeks to reverse this process, generating the basis $\up{B}$ from the dataset $\up{C}$ alone.
In the framework of \glspl{ip} \cite{tarantola05, devito05}, $\up{B}$ could be seen to parameterise a family of mixing models $ψ_{\up{B}} : S^{K-1} \times Λ^K \rightarrow \mathcal{P}$, in this context recovering the correct $\up{B}$ is equivalent to recovering the correct mixing model.

\FloatBarrier
\subsection{Optimal transport primer} \label{OT_sec}

We give a brief overview of a few concepts in \gls{ot},
as they form a basis for the intuition, analysis and the practical implementation of our approach.
It will be useful to highlight the overlap between the Monge and Kantorovich models, discrete and continuous, as well as the \gls{emd} and Wasserstein costs or their minimising maps.
When equivalent, it can be instructive to switch between these different perspectives.
For a deeper overview see \cite{peyre2020computational,Santambrogio2015,Villani2009}.

\begin{figure}[htbp]
  \centering
  \includegraphics[width=\linewidth]{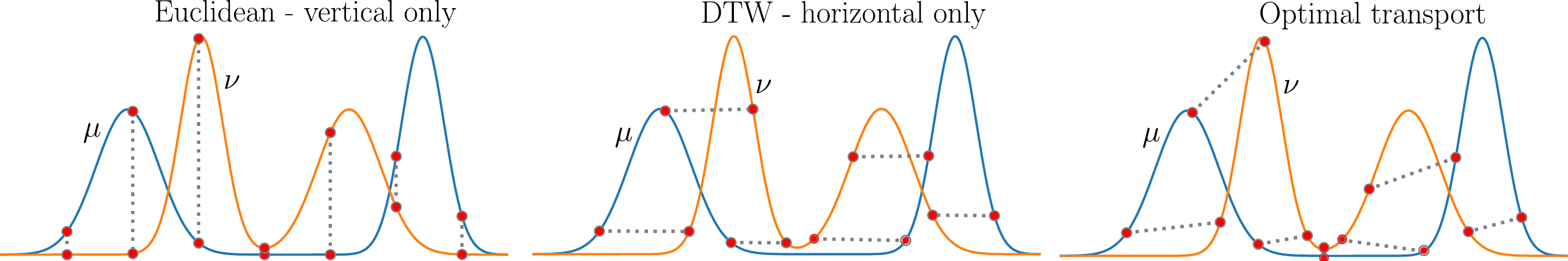}
\caption{
  A distance between two distributions $μ$ and $ν$ can often be written as a sum of distances over parts of $μ$ and $ν$ that it takes to correspond.
  Here, DTW is dynamic time warping.
  Take $μ$ to be an initial pattern (blue) and $ν$ a transformed (orange) counterpart, 
  with corresponding points (red markers) linked by dotted lines.
  In such a sum, change in one part is measured independently of change in another part, 
  as if a point along $μ$ had moved along its dotted line to become the corresponding point in $ν$.
} \label{point_association_fig}
\end{figure}

The problem of \gls{ot} is intuitively expressed in terms of the earth movers analogy.
In this we have a mound of dirt, the shape of which is described by one distribution, and a hole described by another.
We seek a set of movements from the former to the latter which would expend the least overall energy\footnote{The energy/work done in this analogy is proportional to both the ``mass'' of some piece of the mound and to the distance it needs to slide over.}.
While the ``transport'' in this analogy is literal, transport in the general problem should be understood more as association or matching\footnote{And can in fact be thought of as a generalised matching problem.}.

More formally, the \gls{ot} problem aims to find minimum cost associations between two distributions (defined by measures) $μ \in P(\mathcal{X})$ and $ν \in P(\mathcal{Y})$ over \glspl{gs} $\mathcal{X}, \mathcal{Y}$. %
Given a \gls{gs} cost $c : \mathcal{X} \times \mathcal{Y} \rightarrow \mathbb{R}_+$ of transferring mass from $x$ to $y$, an \gls{ot} cost can be formulated as the average \gls{gs} cost when following a particular transport map.
Similar costs appear in information theory as measures of distortion \cite{Pollard82, Lloyd82}.
There are generally two types of maps that are considered, leading to two ways to write this cost.

\begin{figure}[htbp]
  \centering
  \includegraphics[width=0.95\linewidth]{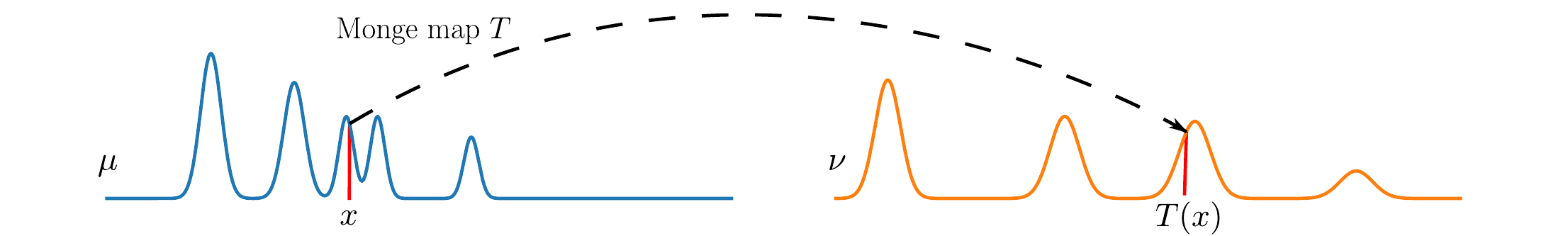}
  \includegraphics[width=0.95\linewidth]{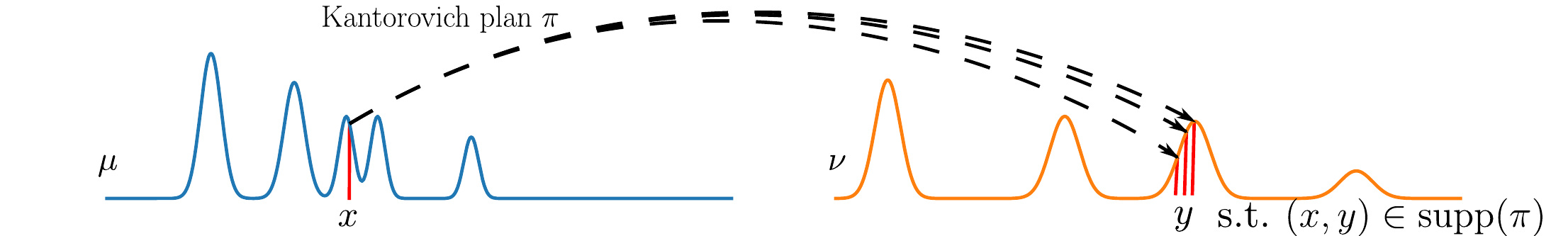}
  \caption{
  While a Monge map seeks a one to one correspondence, 
  with a Kantorovich plan, mass can be broken up and sent to multiple locations.
  This splitting of mass can be necessary in the discrete setting, as a larger ``stick'' (Dirac) may need to be broken into parts and contribute towards multiple others for a mass preserving matching.
} \label{fig:ot_maps}
\end{figure}

\begin{definition}[Pushforward]
  The pushforward of $μ \in \mathcal{M} (\mathcal{X})$ through $T : \mathcal{X} \rightarrow \mathcal{Y}$, results in a measure from $\mathcal{M} (\mathcal{Y})$ denoted $T_\# μ$.
  This is defined as $ (T_\# μ) (Y) = μ( T^{-1} (Y) ) $.
  That is to say, though $T$ redistributes some mass such that it now is spread over $Y$ rather than over $X$, the amounts of mass before and after this redistribution should agree.
\end{definition}

\begin{definition}[Monge cost]
  Provided a Monge map $ T : \mathcal{X} \rightarrow \mathcal{Y} $ that satisfies mass conservation $ T_\# μ = ν $,
  the average \gls{gc} cost may be written
  \begin{equation}
    M^c (T) = \int_\mathcal{X} c(x, T(x)) d μ(x) .
  \end{equation}

\end{definition}

\begin{definition}[Space of couplings]
  Space of distributions $π$ with $μ,ν$ as marginals
  \begin{equation}
    \mathcal{U}(μ, ν) = \{ π \in P (\mathcal{X} \times \mathcal{Y}) : π_\mathcal{X} = μ, π_\mathcal{Y} = ν \} ,
  \end{equation}
where we denote the marginal $ π_\mathcal{X} (X) = \int_{X \times \mathcal{Y}} d π (x,y) $ for $X \subset \mathcal{X}$ and similarly $ π_\mathcal{Y} (Y) = \int_{\mathcal{X} \times Y} d π (x,y) $ for $Y \subset \mathcal{Y}$.
\end{definition}

\begin{definition}[Kantorovich cost]
  Provided a Kantorovich plan $ π \in P(\mathcal{X} \times \mathcal{Y})$ that satisfies mass conservation $ π \in \mathcal{U}(μ, ν) $,
  the average \gls{gc} may be written
  \begin{equation}
    K^c (π) = \int_{\mathcal{X} \times \mathcal{Y}} c(x, y) d π (x,y) .
  \end{equation}
\end{definition}

When $\mathcal{Y}=\mathcal{X}$ and $c(x,y)$ is a distance in $\mathcal{X}$,
the problem can represent literal transport (movement across this common \gls{gs}), as in an \gls{emd}.
Here both the Monge and Kantorovich costs can be interpreted either as the average displacement during a rearrangement of $μ$ into $ν$ 
or equivalently
the energy required by this rearrangement.
In this case, an \gls{ot} distance on $P(\mathcal{X})$ can be defined by taking the minimum cost mass preserving maps
\begin{equation} \label{OT_map_costs}
  d^c_M (μ, ν) = \min_{ T : T_\# μ = ν } M^c (T)
  \quad \text{or} \quad
  d^c_K (μ, ν) = \min_{ π \in \mathcal{U}(μ, ν) } K^c (π) .
\end{equation}

The Monge variant can be written in PDE form by considering the change of variables formula for distributions, where $T$ acts as a reparameterisation.
The Kantorovich variant is a linear functional (in $π$), where the dual problem---$\max_{λ, γ} \int λ (x) d μ(x) + \int γ (y) d ν (y)$ such that $c(x,y) \geq λ(x) + γ(y)$---can be also be used. 

If an optimal Monge map exists, it is equivalent to its Kantorovich counterpart in the sense that $ π^* = (\text{id}, T^*)_\# μ $, (see \cref{monge_kantorovich_equivalence_fig}) where
\begin{calign}
  K_c(π^*) = \int c(x,y) d π^* (x, y) &= \int c(x,y) d ((\text{id}, T^*)_\# μ) (x, y) \\
                                      &= \int c(\text{id} (x), T^* (x)) d μ (x) =
M_c(T^*) .
\end{calign}Or loosely, 
the mass of $π^*$ around the position $(x, T^* (x))$ is equal to that of $μ$ around $x$, and $0$ elsewhere, i.e. all of the mass of $π^*$ is on the graph of the function $T^*$.
\begin{figure}[htbp]
    \centering
    \includegraphics[width=\linewidth]{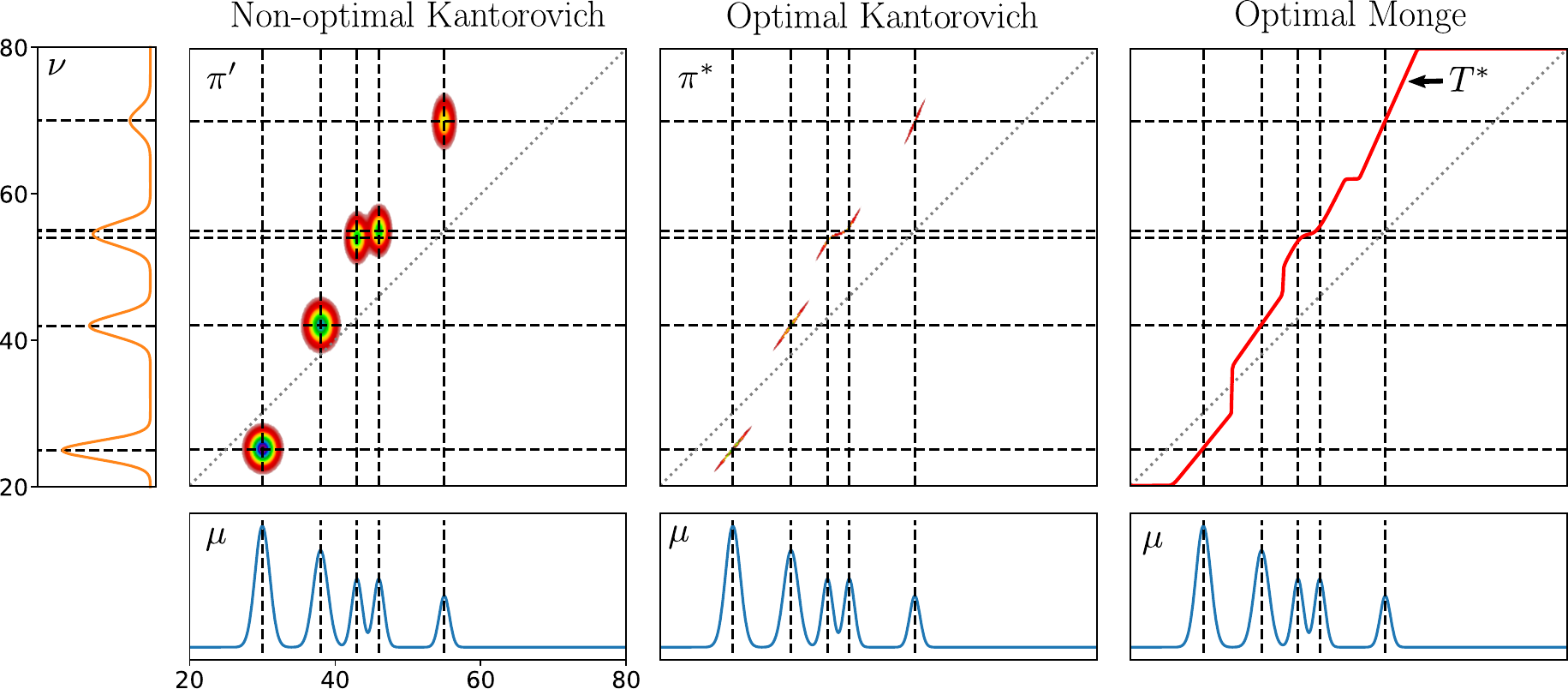}
  \caption{
    Overlap in Monge maps $T$ and Kantorovich plans $π$.
    The density correponding to bivariate distributions $π$ are represented here through heatmaps, it is only noteworthy here where the density is zero (white) and where it is not.
    Nonzero density at $(x,y)$ signifies transport from $x$ to $y$, these $(x,y)$ are in the ``support'' of the Kantorovich plan, $\text{supp}(π^*)$.
    While a Kantorovich plan allows for mass splitting if necessary, it is generally more economical to avoid splitting if possible. 
    Therefore if possible $\text{supp}(π^*) \subseteq \text{graph}(T^*)$.
    Here, the pattern $ν$ was generated by shifting and changing the breadth of the peaks in $μ$.} \label{monge_kantorovich_equivalence_fig}
\end{figure}

Valid distances can also be constructed in the format
\begin{equation} \label{OT_map_costs2}
  W^c_p (μ, ν) = \sqrt[p]{ d^{c^p}_M (μ, ν) }
  \quad \text{or} \quad
  W^c_p (μ, ν) = \sqrt[p]{ d^{c^p}_K (μ, ν) } ,
\end{equation}
these are called Wasserstein distances.
Let the omission of the superscript indicate the usual choice $c(x,y) = \Vert x - y \Vert_p$.

\subsubsection{Discrete optimal transport} \label{discrete_OT}

Let the measure $μ \in P (\mathcal{X})$ have a finite support, so that $ μ = \sum_{i=1}^{N_μ} \sigw{μ}{i} δ_{\sigp[!]{μ}{i}} $,
where we may wish to refer solely to its positions $ \up{X}_μ = \{ \sigp[!]{μ}{i} \in \mathcal{X} : i \in [N_μ] \} $ or its corresponding weights $ \up{W}_μ = \{ w_μ^i : i \in [N_μ] \} $.
The measure $μ$ will be specified via its signature, a tuple of positions and weights $ (\up{X}_μ, \up{W}_μ)$.
The subscripts will be necessary later to distinguish the positions and weights of many different measures.
Sometimes it may be more convenient to refer instead to the weight vector $ \sigw{μ}{} \in \mathbb{R}^{N_μ}_+ $, 
which we will usually take to be normalised, i.e. $\sigw{μ}{} \in S^{N_μ -1}$.

Let\footnote{Here we chose to treat $\mathcal{X}=\mathcal{Y}$ for brevity, but this trivially applies to measures over separate \glspl{gs}.} $ ν \in P (\mathcal{X})$ also be finite with signature $ (\up{X}_ν, \up{W}_ν)$.
The transport plan between these can only be nonzero on $ \up{X}_μ \times \up{X}_ν $, taking the form $π = \sum_{i=1,j=1}^{N_μ, N_ν} p_{ij} δ_{\sigp[!]{μ}{i}, \sigp[!]{ν}{j}}$, hence
\begin{equation}
K^c (π) =
\sum_{i=1,j=1}^{N_μ, N_ν} p_{ij} \int_{\mathcal{X} \times \mathcal{X}} c(x,x') d δ_{\sigp[!]{μ}{i}, \sigp[!]{ν}{j}} (x,x') =
\sum_{i=1,j=1}^{N_μ, N_ν} c (\sigp[!]{μ}{i}, \sigp[!]{ν}{j}) p_{ij}
\end{equation}
and so we arrive at the common form of \gls{emd}
\begin{cgather}
	d_K^c (μ, ν) = \min_{ \mathbf{P} \in \mathbb{R}^{N_μ \times N_ν}_+ } \langle \mathbf{C}, \mathbf{P} \rangle
	\quad \\ \text{s.t.} \quad \quad
	\mathbf{P}^T \mathbf{1} = \sigw{μ}{}
	\quad \text{and} \quad
	\mathbf{P} \mathbf{1} = \sigw{ν}{}
\end{cgather}where $c_{ij} = c (\sigp[!]{μ}{i}, \sigp[!]{ν}{j}) $.

\begin{remark}
  Other applications of \gls{ot} include: 
  a distance between signed measures \cite{Engquist16}; 
  a distance between combinations of elements in any set, as combinations may be treated as distribution, such as in \cite{Hargreaves};
  bootstrapping from intra as opposed to inter \gls{gs} costs, with Gromov-Wasserstein distances;
  and more \cite{peyre2020computational, OTinMLFlamary}.
  Use of the Wasserstein distance has been noted to produce a more stable optimisation in GANs \cite{arjovsky2017wasserstein}.
  \gls{ot} has been used in complex \glspl{ip}, such as the inverse design of glass geometries that produces specific images with its caustics \cite{caustics14}.
\end{remark}

\FloatBarrier
\subsection{Existing foundational techniques} \label{existing_works_sec}

There is a strong association between unmixing and factor based models, and indeed many solutions in which the \textbf{factor representation} is central.
The factor representation of mixing is possible in \textbf{linear mixing}, where the basis is combined linearly, i.e. solely via $κ^\text{lin} $, with no basis variation and taking patterns to be expressed as vectors $ \mathcal{P} \subseteq \mathbb{R}^{N} $. 
This representation usually corresponds to sampling on a shared uniform grid of positions.
Here, multiple patterns (and similarly fractions) can be packed into the columns of a matrix.
Noting that $ \mathbf{c}^j = κ^\text{lin} ({\mathbf{a}^j}; \up{B}) = \mathbf{B} \mathbf{a}^j$, the factor representation follows immediately $\mathbf{C} = ( \mathbf{B} \mathbf{a}^1 \ldots \mathbf{B} \mathbf{a}^M )  = \mathbf{B} \mathbf{A} $,
where $ \mathbf{B} = ( \mathbf{b}^1 \ldots \mathbf{b}^k ) $ and $ \mathbf{A} = ( \mathbf{a}^1 \ldots \mathbf{a}^M ) $.
Since linear mixing can be expressed through matrix multiplication (in $\mathbb{R}^{N}$), linear unmixing, can here be seen as a matrix factorisation problem.

\textbf{Matrix factorisation} techniques such as \gls{pca}, Independent component analysis, and \gls{nmf} can be therefore be seen to perform the required inversion, in the cases of orthogonal basis, independent basis, and nonnegative factor conditions being required respectively \cite{Tingting12,Agrawal22,LeeSeung}.
Accordingly, \gls{nmf} often forms a base for unmixing where nonnegative data is involved.
Of course, while the model being fit may be linear, the fitting itself generally is a nonconvex optimisation.

The model utilised by \gls{onmf} \cite{Cichocki09}, diverges from the ordinary two factor model in \gls{nmf}, by allowing bases to be varied via discrete convolution\footnote{As defined here, it is actually a cross correlation, though the kernel $\bm{λ}$ can be learned reversed.} $ ϕ (\bm{λ} ; \mathbf{b}) = { \bm{λ} *_S \mathbf{b} } = { \sum_{s=1}^{S} λ_s \mathbf{b}^{\downarrow s-1} } $.
By $\mathbf{u}^{\downarrow s} $ we mean the shift-down operation where $ ( \mathbf{u}^{\downarrow s} )_i  = u_{i-s} $ if $i-s>1$ else $0$. 
Correspondingly, $\mathbf{U}^{\downarrow s} $ represents downward shifting each column vector.
This allows for addition of many overlapping shifted copies of each basis in each mixture, hence the name\footnote{Sometimes this may also be referred to by the name of convolutional \gls{nmf} (cNMF), which is a closely related algorithm \cite{Smaragdis}, though relates to a different composition, of 2D objects.}.
Similarly to \gls{nmf}, a solution is sought by minimising a matrix dissimilarity between the mixtures $\mathbf{C}$ and their reconstruction according to the corresponding mixing model.
By observing that the reconstructed mixture can be expressed as
\begin{calign}
\widetilde{\mathbf{c}}^j &= \sum_{i=1}^K a^j_i \bm{λ}^{ij} *_S \mathbf{b}^i = \sum_{i=1}^K a^j_i \sum_{s=1}^S λ^{ij}_s ( \mathbf{b}^i )^{\downarrow s-1} = \sum_{i=1}^K \sum_{s=1}^S \bar{λ}_{ij}^s ( \mathbf{b}^i )^{\downarrow s-1} \\
 &= \sum_{s=1}^S ( \mathbf{B} )^{\downarrow s-1} \bar{\bm{Λ}}^s_{*j}
\end{calign}with 
matrices $ ( \bar{\bm{Λ}}^{s} )_{ij} = a^j_i λ^{ij}_s $,
the usual matrix product sum form of the reconstructed mixtures $ \widetilde{\mathbf{C}} = \sum_{s=1}^{S} \mathbf{B}^{\downarrow s-1} \bar{\bm{Λ}}^s $ is obtained.
This sum of products can be written as a single product of concatenated factors,
hence this model can be seen through the lens of the linear mixing $ \widetilde{\mathbf{C}} = \bar{\mathbf{B}} \bar{\mathbf{A}} $, where 
$ \bar{\mathbf{B}} = (\mathbf{B}^{\downarrow 0}  \ldots  \mathbf{B}^{\downarrow S-1}) $ and similarly $ \bar{\mathbf{A}} = ( \bar{\bm{Λ}}^{0} \ldots \bar{\bm{Λ}}^{S-1} ) $.
Here the basis may be considered over-complete, holding multiple alternative representations of the same basis pattern.
In mixing, only a single representation is desired for each basis pattern and as such sparse fractions are desired.
An instance of this is \gls{afd} \cite{xue2016phasemapper}, where varied bases are generated by a shift.
Though additional constraints would be required so that different elements of the dictionary may correspond to different representations of the same basis, this is not dissimilar to the sparse dictionary learning problem.
Accounts of methods more directly applying a sparse dictionary learning approach can be found in \cite{Borsoi20}.

The \glspl{ip} of \gls{nmf} and \gls{onmf} are usually tackled through \textbf{reconstruction error minimisation}, through costs such as $ \sum_j D ( ψ (\mathbf{a}^j, \bm{Λ}^j; \up{B}), c^j ) $ where $D$ is a distance, dissimilarity or divergence, for example, $\sum_j \Vert \mathbf{B} \mathbf{a}^j - \mathbf{c}^j \Vert_2^2$.
Recall that while all three metric axioms are satisfied by a distance, dissimilarities and divergences may violate the triangle inequality and divergences may also violate symmetry.
A reconstructed data matrix $ \widetilde{\mathbf{C}} $ can also be defined, where the $j^\text{th}$ column is given by the reconstructed mixture $ ψ (\mathbf{a}^j, \bm{Λ}^j; \up{B}) $ and with the cost $ D (\widetilde{\mathbf{C}}, \mathbf{C} ) $ being minimised instead, for some matrix divergence, for example $\Vert \widetilde{\mathbf{C}} - \mathbf{C} \Vert_2^2$, or column-wise \gls{emd} such as in EMD-NMF\footnote{
Note, while EMD-NMF performs a more robust fit of the linear mixing model, it still only fits a linear/two factor model.
} \cite{Sandler11}.
The Wasserstein barycenter fitting approach \cite{Schmitz18}, also forms such a cost.

A cost of course needs to act as a notion of proximity to a desirable solution, however, there are two reasons why a cost of the format described above may not do this as well as possible.

\textit{A good fit of the mixture may not require a good basis.} If the bases may be arbitrarily contorted, as is the case in using an \gls{ot} model for basis transformation, a perfect fit may be achieved with any choice of basis, and so does not serve to guide us to a plausible basis.

\begin{figure}[htbp]
  \centering
  \includegraphics[width=\linewidth]{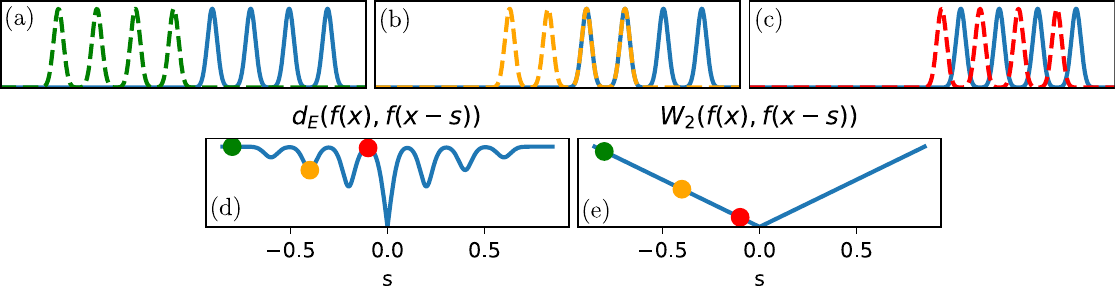}
\caption{
  Distances after offsetting an original density $f$ (a-c, solid blue), producing shifted counterparts (a-c, dashed) for various shifts $s$, the markers in panels (d) and (e) are colored according to corresponding values of $s$.
  Panels (d) and (e) track a pointwise distance (Euclidean) and an \gls{ot} distance (Wasserstein) respectively, between the original and a shifted copy, as a function of shift.
  Note, since a pointwise distance only compares features that align vertically (see \cref{point_association_fig}), we obtain local minima whenever non-corresponding peaks happen to align (such as in the yellow case).
  Note also the plateau in (d) when no features align (green case), where the gradient is near $0$, and provides little direction in optimisation (even after perturbation).
} \label{shift_convexity_fig}
\end{figure}

\textit{Avoidable local minima.}
As well as local minima corresponding to alternative solutions, a typical reconstruction cost can also be stationary at misaligned copies of the same solution.
This is due to the use of pointwise distances and the many locally optimal comparisons that they can produce.
Suppose $ \mathbf{A} $ and $\up{B}$ are fixed.
Even if there exists a unique globally optimal $\bm{Λ}^j_*$ with $ ψ (\mathbf{a}^j, \bm{Λ}^j_*; \up{B}) = c^j $, 
there can still be other $ \bm{Λ}^j $ that locally minimise $ F^j (\bm{Λ}^j) =  D (ψ (\mathbf{a}^j, \bm{Λ}^{j}; \up{B}), c^j) $.
For instance, consider registering a single uniformly shifted basis in a pointwise distance (see \cref{shift_convexity_fig} - bottom left).
Any combination of local minima of $F^j$ is a local minimum of $ \sum_{j=1}^M F^j( \bm{Λ}^j ) $,
hence if each $F^j$ has $L$ local minima, the sum has $ L^M $ local minima.
If we have $100$ mixed examples, and the alignment of each is akin to \cref{shift_convexity_fig}, then each feasible decomposition has roughly an additional $3.23 \times {10}^{84}$ misregistered counterparts, entirely due to the choice of distance used in construction of a fitting cost.

\subsubsection{Geometric approaches} \label{sec:geom_methods}

Not only can the image of a mixing model $ψ$, and hence the point cloud of generated mixtures,
form a specific structure within pattern space $\mathcal{P}$, but inputs such as the basis may be reflected through this structure as geometric features.
Geometric approaches recognise this and look for such geometric features.
We give these two classifications: those which consider the spatial structure arrising from a particular form of the mixing model $ψ$, 
which strongly informs which geometric features should be sought, 
and those which seek some analogue to features present in the linear case, 
without necessarily considering whether the presence of such analogues will be preserved under the change in context.

In the linear case, taking all convex combinations of $b^i$ simply forms a convex hull\footnote{This is a simplex given a number of bases not more than one above the dimensionality of the patterns.} of these patterns, this contains all possible mixtures in the normalised case, with the bases at the vertices of the hull (\cref{abstract_mix_fig} - left).
Supposing a dataset provides a dense enough sampling of the possible mixtures, 
then to find the basis one needs only identify vertices.
Simplex structure forms the foundations for simplex methods for linear mixing \cite{Ma14}.
A modified structure may be obtained if one first allows bases to be perturbed prior to combination (\cref{abstract_mix_fig} - middle) \cite{Stein03,Du14,Zhou18}.
While this modification considers variation of basis patterns,
the distributions considered only assign significant likelihood to basis variation that does not drastically increase Euclidean distance (which can be large for non-pointwise variation).

Some approaches seek a simplex embedded within a manifold, 
for instance through search of vertices via counting graph geodesics \cite{Heylen14}.
This offers an analogue for simplex fitting in the non-linear case, however, as basis variation is not treated independently in this model, it would not be able to inspect how each basis has changed individually, in the context of a given mixture.
It is also not clear what kinds of mixing might produce such structure.
Manifold analogues of \gls{pca} have also been explored \cite{Fletcher04}, 
however a manifold analogue of \gls{nmf} might be more appropriate for non-negative data.

\FloatBarrier
\section{Latent basis variant clustering} \label{main_contribution}

\begin{figure}[htbp]
  \centering
  \includegraphics[width=.98\textwidth]{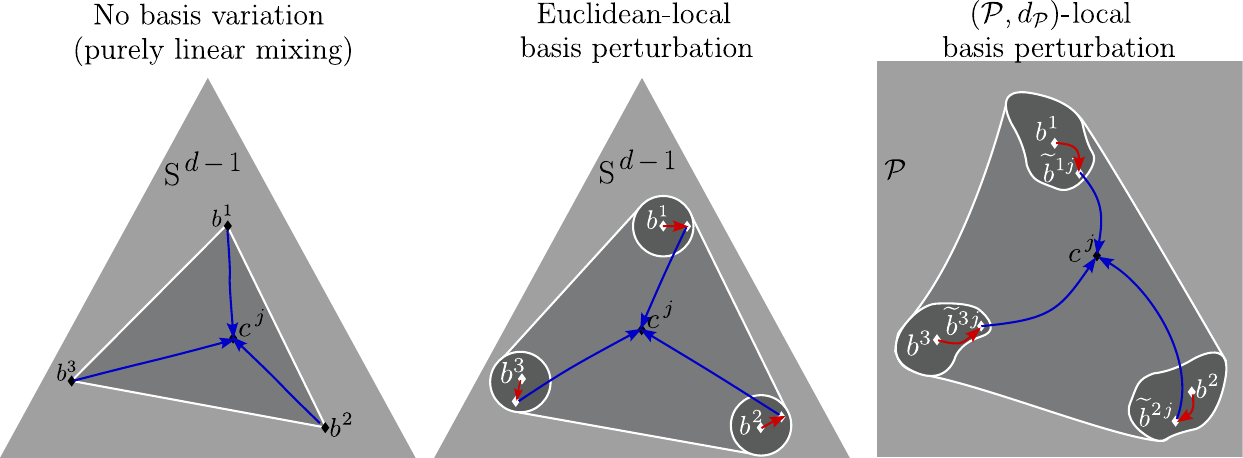}
  \caption{
    Geometric views of mixing under three kinds of basis variation (see \cref{basis_var_fig}) and linear combination, with the action of the variation $ϕ$ and combination $κ$ functions denoted with red and blue arrows respectively.
    Dark grey represents possible variations $\widetilde{b}$ of basis pattern $b$, lighter grey - possible combinations, and lightest - the whole space of patterns; the simplex $ \up{S}^{d-1} = \{ \mathbf{x} \in \mathbb{R}^d_+ : \Vert \mathbf{x} \Vert_1 = 1 \} $ or an arbitrary pattern space $\mathcal{P}$.
  } \label{abstract_mix_fig}
\end{figure}

Consider the space of mixing inputs (fractions and basis variations) that combine to produce a given arbitrary pattern $ p \in \mathcal{P} $ exactly
\begin{equation}
  \mathcal{I} (p) = \{ ({\mathbf{a}}, \widetilde{\up{B}}) \in S^{K-1} \times \mathcal{P}^K : \kappa ({\mathbf{a}}; \widetilde{\up{B}}) = p \} .
\end{equation}
A decomposition that perfectly reconstructs all of our given mixtures can be obtained by simply selecting one element from each $\mathcal{I} (c^j)$.
Such a selection of $\widetilde{b}^{ij}$ may have no particular spatial structure (or connection to a shared basis), however under mild conditions (see \cref{separability_sec}), we expect clustering of basis variants.
There may be many solutions in $ \prod_j \mathcal{I} (c^j) $ which appear to be clustered,
a simple way to find one would be to optimise a clustering objective such as minimising in class variation.
This has the desirable side effect of favouring bases requiring less modification on average to explain the data, which we might say provide a more natural explanation.
To measure variation, we will employ the Fréchet variance.

The weighted Fréchet variance of a cluster of patterns $ \widetilde{\up{P}} \subset \mathcal{P}$ is defined as follows
\begin{equation} \label{frechet_var}
  \text{FV} (\widetilde{\up{P}}, \mathbf{w}) = \min_{ p \in \mathcal{P} } \sum_{j=1}^{\vert \widetilde{\up{P}} \vert} w_j d_\mathcal{P} ( p, \widetilde{p}^j )^2 .
\end{equation}
Much like ordinary variance, it measures average squared deviation from the mean (defined by the minimiser), with the modification that deviation is given by $d_\mathcal{P}$ rather than the Euclidean distance.

Our problem is therefore
\begin{equation} \label{var_prob_2}
  \min_{  (\mathbf{a}^j, \widetilde{b}^{*j}) \in \mathcal{I} (c^j) \: \forall j} \sum_{i=1}^{K} \text{FV} (\widetilde{b}^{i*}, a_i^*)
  = \argmin_{ \substack{ (\mathbf{a}^j,\up{\widetilde{B}}^j) \in \mathcal{I} (c^j) \: \forall j \\
              b^i \in \mathcal{P} \: \forall i
  }} \sum_{i=1}^K \sum_{j=1}^M  a_i^j d_\mathcal{P} ( b^i, \widetilde{b}^{ij} )^2
\end{equation}
where we weight according to abundances, as we wish components that play a lesser role in a given mixture, to have less overall influence.
This is of course much too large a problem in practice, however, it can be decoupled into many smaller subproblems through alternating optimisation.

\textbf{Averaging} - given $ \widetilde{b}^{**} $, update cluster centroids
\begin{equation} \label{average_prob}
    \forall i \in [K] \quad \quad
    b^i = \argmin_{b \in \mathcal{P}} \sum_{j=1}^M a_i^j d_\mathcal{P} ( b, \widetilde{b}^{ij} )^2 . 
\end{equation}

\textbf{Separation} - given $ b^* $, if possible, select feasible mixing inputs that are less spread out from the cluster centers
  \begin{equation} \label{separation_prob}
    \forall j \in [M] \quad \quad
    \mathbf{a}^j, \widetilde{\up{B}}^j = 
    \argmin_{ (\mathbf{a}, \up{\widetilde{B}}) \in \mathcal{I} (c^j) } \sum_{i=1}^K  a_i d_\mathcal{P} ( b^i, \widetilde{b}^{i} )^2
 .
  \end{equation}

Clearly \cref{var_prob_2} is non-increasing under these alternating minimisations, as each updates a subsets of the overall sum in a non-increasing fashion.
Note, either the cost is strictly decreasing, in which case we have convergence as the cost is bound below by 0, 
or neither averaging or separation have updated any terms at some iteration, in which case no further updates can be made, also converging.

For data such as that produced by \gls{pxrd}; perturbation mechanisms such as shifting and broadening are of particular importance.
While pointwise distances do not behave well with respect to these (see \cref{shift_convexity_fig}),
Wasserstein distances are convex (and therefore locally Lipschitz) to these and more \cite{Engquist16}, and so we can expect clustering of basis variations (see \cref{separability_sec}).
Accordingly, from here we consider the set of all distributions $\mathcal{P} = P (\mathcal{X}) $ on a geodesic \gls{gs} ($\mathcal{X}$, $d_\mathcal{X}$).
The distance on $\mathcal{P}$ will be the Wasserstein distance $ d_\mathcal{P} = W_2^{d_\mathcal{X}} $.

\subsection{Averaging} \label{avg_sec}

Let $(S, d_S)$ be some metric space.
For convenience, let the weighted centroid of two points $μ,ν \in S$ be expressed by
\begin{equation}
  B_{d_S} (μ, ν, λ) = \argmin_{β \in S} \left[ (1-λ) d_S ( β, μ )^2 + λ d_S ( β, ν )^2 \right] ,
\end{equation}
this can in fact be obtained by walking a fraction of the way along the geodesic connecting the two points $g_{μ}^{ν} : [0,1] \rightarrow S $, or $ B_{d_S} (μ, ν, λ) = g_{μ}^{ν} (λ)$ (see \cref{average_along_geodesics}).
In the \gls{ed} case, where $d_S (μ, ν)= \Vert μ - ν \Vert_2$, the centroids lie on the line $ (1-λ) μ + λ ν $.
Note, this connection between centroids and geodesics comes in handy both in the ground space and in the overarching Wasserstein space.

\begin{definition}[The general McCann interpolation]
  Let $ T^λ (x) = B_{d_\mathcal{X}} (x, T^* (x), λ) $ where $T^*$ is the optimal transport map from $μ$ to $ν$, then the general McCann interpolant is the pushforward $T^λ_\# μ$ (for $λ \in [0,1]$).
\end{definition}

A Monge map can tell us that mass at $x$ should go to $T(x)$, but does not specify how this should happen.
In the McCann, or displacement interpolation, we continuously (and at constant speed) move mass along the \gls{gs} geodesic connecting $x$ and $T(x)$, obtaining a measure somewhere between $μ$ and $ν$ at each instant (see \cref{displacement_interp_vis}).
Since the general McCann interpolant follows the Wasserstein space geodesic (see \cref{general_mccan_interp}), just as the centroid does (see \cref{average_along_geodesics}), one can use the McCann interpolant to calculate centroids in Wasserstein space and vice versa
\begin{equation}
	B_{W_2^{d_\mathcal{X}}} (μ, ν, λ) = T^λ_\# μ.
\end{equation}

When a Monge map does not exist, an approximation---in the sense $ W_2(\overline{T}_\# μ, ν) \approx 0 $---can be generated via the barycentric projection, $\overline{T}(x) = \int_\mathcal{X} y d π^* (x, y)$ in the \gls{ed} \gls{gc} case.
Thanks to the displacement interpolation, there is no need to solve a coupled transport problem between in $ π^1 \in \mathcal{U} (μ, β(λ))$ and $ π^2 \in \mathcal{U} (β(λ), ν) $ for each weight $λ$, 
knowing just $ π \in  \mathcal{U} (μ, ν) $ gives us $β(λ) $ for all $ λ \in [0,1]$.

\begin{algorithm}[htb]
  \caption{Pairwise barycenter} \label{pairwiseBary}
 \begin{algorithmic}%
\Function{$\bar{B}_{W_2^{d_\mathcal{X}}}$}{$μ$, $ν$, $λ$}
\Comment{Approximation due to projection better for lower $λ$.}

\State $ \mathbf{P} \gets$ \Call{$\text{deriveOTP}_{d_\mathcal{X}}$}{$μ$, $ν$} \Comment{Choose $\text{deriveOTP}$ as appropriate, we consider the \gls{nwca} later}

 \For{$i$ from $1$ to $N$}

 \State $ \sigp{ν'}{i} \gets B_{d_\mathcal{X}} (\up{X}_ν, \mathbf{p}_{i*} )$  \Comment{Barycentric projection, $\sum_{j=1}^M p_{ij} \sigp{ν}{j} / \sigw{μ}{i}$ for \gls{ed}}

 \State $ \sigp{β}{i} \gets B_{d_\mathcal{X}} (\sigp{μ}{i}, \sigp{ν'}{i}, λ) $ \Comment{Displacement interpolation, $ λ \sigp{μ}{i} + (1-λ) \sigp{ν'}{i}$ for \gls{ed}}

 \EndFor

\State \Return $(\up{X}_β, \up{W}_μ)$

\EndFunction

 \end{algorithmic}
\end{algorithm}

When averaging more than two patterns,
we approximate barycenters hierarchically, as in \cite{Boissard11}, see \cref{mBaryAlg}.

\begin{figure}[htbp]
  \centering
  \includegraphics[width=0.98\linewidth]{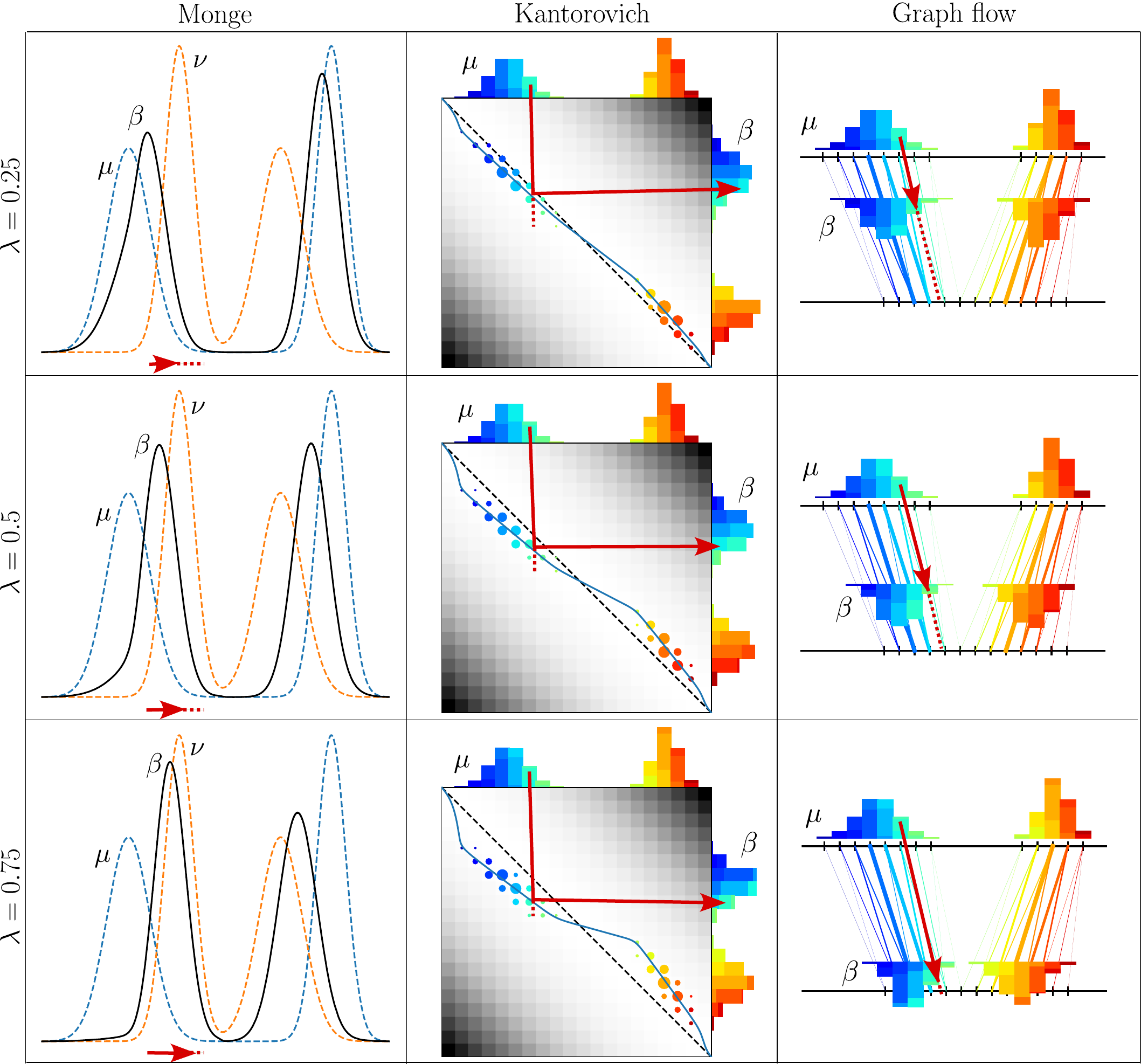}
  \caption{
    The displacement interpolation (\gls{ot} barycenter construction) from three perspectives (one per column).
    The barycenter captures aspects of multiple patterns by striving to be `in-between' them, attempting to be similar to both.
    With an appropriate distance, this can combine information in a meaningful way.
    The barycenter $β$ will be more $μ$-like or $ν$-like according to the weighting coefficient $λ$.
    Ordinarily a transport map associates some mass at position $x$ in $μ$ to mass at position $y$ in $ν$ (see \cref{fig:ot_maps}).
    With this association, we can consider continuously moving this mass from its location within $μ$ to its location within $ν$, along the shortest path between $x$ and $y$.
    With $λ=0$, all mass elements are at their source position and constitute $μ$.
    With $λ=1$, all mass elements are at their target position and constitute $ν$.
    Mass elements in intermediate positions, constitute an intermediate pattern $β$.
    In latter columns, a segment of one colour in $μ$ corresponds to that of the same color in $β$ and in $ν$.
    The blue curve in the Kantorovich column illustrates the Monge map from $μ$ to $β$.
  } \label{displacement_interp_vis}
\end{figure}

\FloatBarrier
\begin{algorithm}
  \caption{Hierarchical Wasserstein centroid approximation} \label{mBaryAlg}
 \begin{algorithmic}%

	\Function{HWCA}{$\up{W} = (w^1, \ldots, w^{N_M}) $, $\up{M} = (μ^1, \ldots, μ^{N_M}) $}

	 \Comment For the sake of brevity, let $N_M = \vert \up{M} \vert$ be a power of two.

   \Comment Take $\up{W}$, $\up{M}$ to be sorted such that weights are decreasing.

	 \While{ $ \vert \up{M} \vert > 1 $ }

	 \State $ \up{M}' \gets (), \up{W}' \gets () $

		 \For{$k$ from $1$ to $ \frac{\vert \up{M} \vert}{2} $} \Comment{Can be run in parallel}

          \State $ w_k' \gets w_{2 k - 1}+w_{ 2 k } $

				  \State $ μ_k' \gets $ \Call{$\bar{B}_{W_2^{d_\mathcal{X}}}$}{$ μ^{2 k - 1}, μ^{ 2 k }$, ${w_{2 k - 1}}/{w_k'} $}

		 \EndFor

		 \State $ \up{M} \gets \up{M}', \up{W} \gets \up{W}' $ \Comment{Note, these can be edited in place}

	 \EndWhile

	 \State \Return $μ^1$

\EndFunction

 \end{algorithmic}
\end{algorithm}

\FloatBarrier
\subsection{Separation} \label{wasserstein_separation_sec}

Recall the notation of \cref{discrete_OT}.
Discretising and writing in terms of abundance scaled versions of the transport maps $ \bar{\mathbf{P}}^{i} = a_i \mathbf{P}^i $ between $b^i$ and $\widetilde{b}^i$, the discrete $W_2$ separation subproblem (see \cref{separation_prob}) becomes
\begin{cgather} \label{separation_prob_discrete}
\min_{ \bar{\mathbf{P}}^1 , \ldots , \bar{\mathbf{P}}^K \in \mathbb{R}^{N \times N}_+ }
 \sum_{i=1}^{K} \langle \mathbf{D}^{i}, \bar{\mathbf{P}}^{i} \rangle \\
  \text{s.t.} \quad
  \left( \bar{\mathbf{P}}^{i} \right)^T \mathbf{1}  = { \Vert \bar{\mathbf{P}}^{i} \Vert_1} \sigw{b^{i}}{} \quad \forall i \in [K]
  \quad \text{and} \quad
  { \sum_{i=1}^K \bar{\mathbf{P}}^{i} \mathbf{1}} = \sigw{c}{}
\end{cgather}where 
the entries of $\mathbf{D}^{i}$ are given by $d^{i}_{kl} = d_\mathcal{X} \left( \sigp{b^{i}}{k}, \sigp{c}{l} \right)^2 $ and all measures have been taken to have the same size support for brevity\footnote{Uneven measures can be trivially padded with zero weight Diracs.}. 
The combination constraint
\begin{equation}
\sum_{i=1}^K a_i \widetilde{b}^{i} = c \implies \sum_{i=1}^K a_i\sum_{k=1}^N \sigw{\widetilde{b}^{i}}{k} δ_{\sigp{\widetilde{b}^{i}}{k}} = \sum_{k=1}^N \sigw{c}{k} δ_{\sigp{c}{k}}
\end{equation}
gives that the varied basis patterns cannot have weight where the mixture does not or $ \text{supp} (\widetilde{b}^{i}) \subset \text{supp} (c) $,
hence we can set $ \sigp{\widetilde{b}^{i}}{k} = \sigp{c}{k} $.
And so, during the separation subproblem, minimisation with respect to weights alone is sufficient.
The weight vector $\sigw{\widetilde{b}^{i}}{}$ is obtained from the marginal $ \bar{\mathbf{P}}^{i} \mathbf{1} /\lVert \bar{\mathbf{P}}^{i} \rVert_1$.

\begin{figure}[htbp]
  \centering
  \begin{subfigure}{.46\textwidth}
  \centering
    \includegraphics[width=\linewidth]{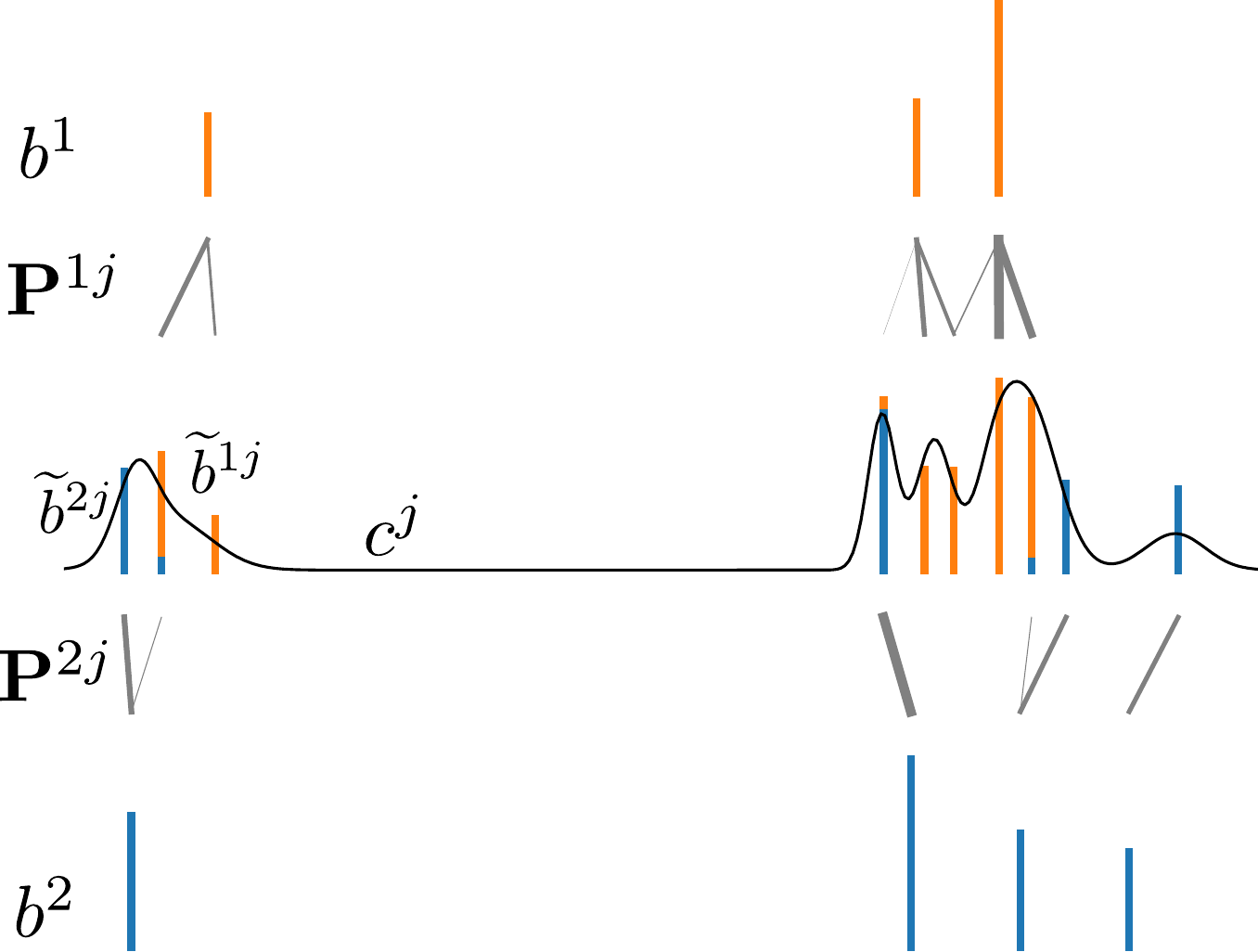}
    \caption{Coupled multi-source transport problem, with individual transport plans $\mathbf{P}^{ij}$ between bases $b^i$ and variations $\widetilde{b}^{ij}$, in each mixed datapoint $c^j$.}
  \end{subfigure}
  \hspace{0.1cm}
  \begin{subfigure}{.46\textwidth}
    \centering
    \includegraphics[width=\linewidth]{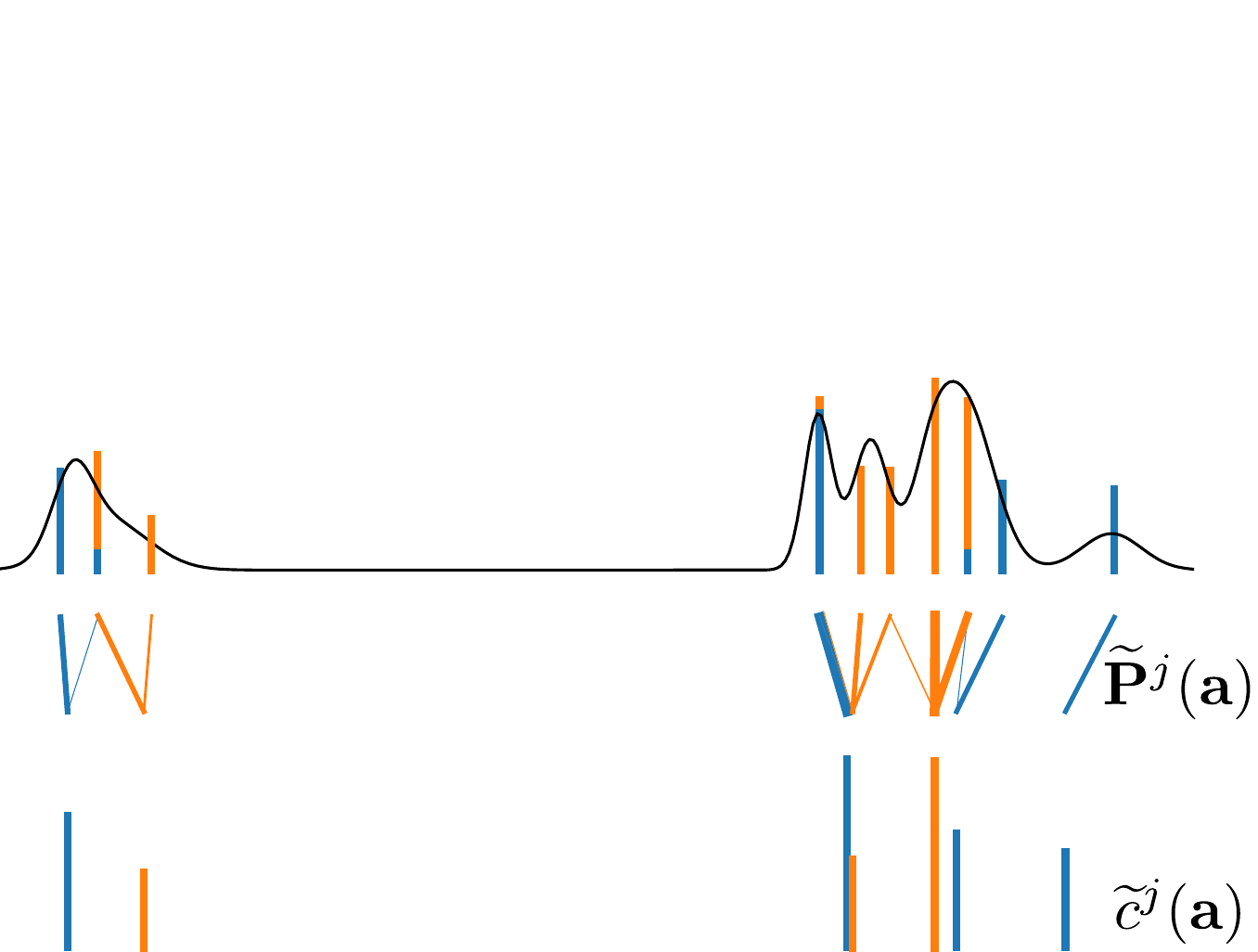}
    \caption{Single transport plan $\tilde{\mathbf{P}}^j$ between mixture $c^j$ and concatenated sources $ \widetilde{c}^j = \sum_{j=1}^K b^j $.\\\phantom{a}} %
  \end{subfigure}
  \caption{
    The segmentation problem as it emerges from the variation cost in the Wasserstein context (a), minimising with respect to multiple coupled plans, and an equivalent, single transport plan/distance evaluation simplification (b).
    Note, in the former case we inherently discover the abundances $\mathbf{a}$ when optimising for unnormalised $b^j$, 
    but in the latter case we must evaluate the distance/plan for each choice of $\mathbf{a}$.
    However, since we can construct exact plans for the ordinary 1D optimal transport problem easily,
    now we only have to optimise for $\mathbf{a}$.
  } \label{concat_set_fig}
\end{figure}
Concatenating
\begin{calign}
  \widetilde{\mathbf{D}} &= ( \mathbf{D}^1 \mathbf{D}^2 \ldots \mathbf{D}^K ) \\
  \widetilde{\mathbf{P}} &= ( \bar{\mathbf{P}}^1 \bar{\mathbf{P}}^2 \ldots \bar{\mathbf{P}}^K ) \\
  \sigw{\widetilde{c}}{}^T &= ( a_1 \sigw{b^1}{}^T \ldots a_k \sigw{b^k}{}^T ) , \\
\end{calign}we can 
now express the separation as
\begin{cgather} \label{separation_prob_concat}
\min_{ \mathbf{a} \in \up{S}^{K-1} } \langle \widetilde{\mathbf{D}}, \widetilde{\mathbf{P}} \rangle
\quad \\ \text{s.t.} \quad \quad
\widetilde{\mathbf{P}}^T \mathbf{1} = \sigw{\widetilde{c}}{}
\quad \text{and} \quad
\widetilde{\mathbf{P}} \mathbf{1} = \sigw{c}{}
\end{cgather}where 
$ \widetilde{c}(\mathbf{a}) = κ^\text{lin} (\mathbf{a} ; \up{B}) $, as supposed to $ c = κ^\text{lin} (\mathbf{a} ; Φ \up{B} ) $, so we can in fact write
\begin{equation} \label{minDistSeg}
	\min_{\mathbf{a} \in \up{S}^{K-1}} W_2^{d_\mathcal{X}} ( \widetilde{c}(\mathbf{a}), c )^2
\end{equation}
which is a convex optimisation problem that benefits from being able to efficiently evaluate $W_2^2$ and its optimal transport map.

For the initial retrieval of abundances, this optimisation is indeed equivalent to $ W_2^{d_\mathcal{X}} $ regression in the linear model, 
such as in \cite{Sandler11}, however, we do not stop there (see \cref{separationAlg}),
as we are not after a linear representation,
we are after the transformed basis patterns inferred from the unconstrained marginals.
If the initial distance computation cannot expose the concatenated transport plan, 
we must use another method to recover it afterwards.
We then separate this into individual plans (see \cref{concat_set_fig}).
Note the advantage in non-overlapping supports here, where some positions can be associated with a particular basis pattern and therefore transport to/from such positions in the combined map can be identified as transport to/from the corresponding pattern.

\begin{algorithm}
  \caption{\gls{ot} separation (unmixing with known basis)} \label{separationAlg}
 \begin{algorithmic}%

 \Function{separate}{$\up{B}$, $c$}

 \State $ \mathbf{a} \gets \argmin_{\bar{\mathbf{a}}} W_2^{d_\mathcal{X}} \left( \widetilde{c}(\bar{\mathbf{a}}), c \right)^2  $  \label{separationAlg:minLine}

   \State $ ( \bar{\mathbf{P}}^1 \bar{\mathbf{P}}^2 \ldots \bar{\mathbf{P}}^K ) \gets $ \Call{$\text{deriveOTP}_{d_\mathcal{X}}$}{$ \widetilde{c}(\mathbf{a})$, $c$}

   \For{$i \in \{1, \ldots, K \} $}

      \State $ \sigw{\widetilde{b}^{i}}{}^T \gets \bar{\mathbf{P}}^{i} \mathbf{1} /\lVert \bar{\mathbf{P}}^{i} \rVert_1 $

   \EndFor

 \EndFunction

 \end{algorithmic}
\end{algorithm}

\FloatBarrier
\subsection{Proposed unmixing algorithm}

Given a method for efficiently computing transport plans (see \cref{monotonic_maps_1D} for 1D) and provided that centroids can be computed (or well approximated) in the given \gls{gs} (in a feasible timescale), such as via a hierarchical approach, we arrive at the general unmixing algorithm \cref{unmixingAlg}.
This algorithm is highly parallelisable\footnote{The averaging step is broken in to a series of pairwise reduction passes, these can be run in parallel.}, as due to the degree of decoupling in the problem, there are many independent subproblems and for large datasets, GPUs can be leveraged to solve a larger subset of these concurrently.
We provide a concrete implementation for the $(P(\mathbb{R}), W_2)$ case.

\begin{algorithm}
  \caption{Warping robust unmixing for discrete measures} \label{unmixingAlg}
 \begin{algorithmic}%

	\Function{W2unmix}{$C$, $N^{\text{iter}}$, $K$}

	\State Generate random initial $\widetilde{\up{B}}$ (random weights over each $x \in \text{supp} (c)$)

	 \For{$k \in \{ 1, \ldots, N^{\text{iter}} \} $}

			 \For{$i \in \{1, \ldots, K \} $} \Comment{Averaging step}

				 \State $b^i \gets $ \Call{HWCA}{$\mathbf{a}_i$, $ ( \widetilde{b}^{i1}, \ldots, \widetilde{b}^{iM} ) $}

			 \EndFor

			 \For{$c \in C$} \Comment{Can be performed independently per mixture}

			 		\State \Call{separate}{$\up{B}$, $c$}

			 \EndFor

	 \EndFor

	 \State \Return $\widetilde{\up{B}}, \up{B}$

 \EndFunction

 \end{algorithmic}
\end{algorithm}

\FloatBarrier
\subsection{Variations of the unmixing problem} \label{prob_variants_sec}

Depending on what information is originally known, the problem of interest may in fact be more constrained.
For instance our separation step provides a solution to the problem of separating components when the basis is known.
For partially blind unmixing, the known patterns are included in the initial basis instead of fully random data and can be kept fixed by skipping the corresponding basis update.
However, if we wish to detect novel patterns, 
as one often may in materials discovery,
these will not have comparably weighted clusters in our geometry,
as these patterns are effectively outliers and are drowned out by the more common bases.
As such, we cannot simply use the partially blind approach in this case.
A dedicated anomaly detection approach is necessary.
\begin{table}[hbt]
    \begin{tabular}{ | c | c | L{0.4\textwidth}| }
    \hline
    Problem & Amount of basis known  & Primary goal\\
    \hline
     Separation (non-blind) & All & For each mixture, determine parts likely arising from each basis \\
    \hline
     Partially blind unmixing & Some & Find remaining basis patterns \\
    \hline
     Fully blind unmixing & None & Find basis $B$ \\
    \hline
    \hline
    + Anomaly detection & - & Identify datapoints containing trace uncommon patterns \\
    \hline
    \end{tabular}
  \caption{Table of variations of the unmixing problem. 
           The anomaly detection relaxation can be optionally made from any of the first three variants,
           it is treated separately from partially blind unmixing since a sought additional component may be too statistically insignificant to estimate (though may be indirectly detected).} \label{variants_table}
\end{table}

The presence of sufficiently novel patterns can still be detected through our usual formulation, via an uncharacteristically large segmentation cost for a particular mixture.
This can be seen as attempting to fit a mixture of $K+1$ distinct basis patterns with only a $K$ subset, will involve additional deformation of these $K$ basis patterns.
By constraining the deformation of these basis patterns, to less than what would minimally be required for fully fitting with a $K$ subset, we can attempt to estimate the unknown through the residual, after a best partial fit.
Again, residuals need only be enabled for those mixtures where anomalies are reasonably suspected.

This changes the separation problem to the following linear program
\begin{cgather}
\min_{ \substack{ \bar{π}^{i} \in M_+ (\mathbb{R} \times \mathbb{R}) \\ \bar{η} \in M_+ (\mathbb{R}) } }
 \sum_{i=1}^K \langle d, \bar{π}^{i} \rangle + β \Vert \bar{η} \Vert_1 \\
\text{s.t.} \quad
(\text{proj}_x)_\# \bar{π}^{i} = \Vert \bar{π}^{i} \Vert_1 b^i  \quad \forall i \in [K]
, \quad
(\text{proj}_y)_\# \sum_{i=1}^K  \bar{π}^{i} + \bar{η} = c \label{anom_seg} \\
\quad \text{and} \quad
\langle d, \bar{π}^{i} \rangle \leq \overline{Δ^2}
\end{cgather}where 
the $\Vert \bar{η} \Vert_1 = \bar{η} (\mathbb{R}) = ε $ penalty discourages unnecessarily attributing mass to $η$, and $\overline{Δ^2}$ is a specified maximum average squared displacement, set to a value less than the minimum of the unconstrained segmentation cost.
In the discrete linear program, we can alternatively (or additionally) enforce a maximum displacement constraint, this can be achieved efficiently by disregarding variables $\bar{p}^{i}_{kl}$ that correspond to $d^{i}_{kl} > Δ^2$, reducing the amount of variables that need to be solved for.

\FloatBarrier
\section{Fast Wasserstein} \label{fast_wasserstein_sec}

Here we will cover two ways in which to exploit structure in the \gls{ot} problem, for a fast exact and an approximate distance calculation.

\subsection{The north west corner algorithm} \label{monotonic_maps_1D}

\begin{definition}[Monotone increasing] \label{monotone_increasing_def}
  A transport plan $π$ is monotone increasing if
  \begin{equation}
    \langle x' - x, y' - y \rangle \geq 0
    \quad \forall (x, y), (x', y') \in \text{supp} (π) .
  \end{equation}
\end{definition}

In the 1D case, the monotonicity of transport plans fully constrains them, such that no minimisation is required.
Further, for \glspl{gc} that admit monotone increasing solutions, such as the \gls{ed}, this unique feasible map can be constructed through the \gls{nwca} (in $ \mathcal{O} (N)$ time).
Though this fact is known in \gls{ot} literature \cite{bai2023sliced}, it does not appear to be always exploited downstream.

\begin{lemma} \label{trace_monotonicity}
  Though a Monge map is not sufficient to convey optimal transport between a pair of general measures (such that $\text{supp} (π^*) \subseteq \text{graph} (T^*)$), 
  a parameterised trace $ (x,y) = (T_x (z), T_y (z)) $\allowbreak---with monotone $T_x, T_y$---is in 1D with the \gls{ed} \gls{gc}. 
Similarly, in the discrete case, a sequence $ ( \sigp[!]{μ}{i}, \sigp[!]{ν}{j}) $ with monotone  $\sigp[!]{μ}{i}$ and  $\sigp[!]{ν}{j}$ is sufficient.
\end{lemma}

\begin{proof}
  In the discrete case we have the following.
  Take the points in the support of the optimal (Kantorovich) map in any initial order,
  forming the sequence $ z^i = (x^i, y^i) \in \text{supp} (π^*) $.
  Suppose there are no $z^i$ with the same $x$ value, then we can order these such that $\{x^i\}$ are strictly increasing. 
  Since the optimal plan is monotone increasing $π^*$, $(x^{i+1} - x^i)(y^{i+1} - y^i) \geq 0$ must hold, where $x^{i+1} - x^i > 0$ by ordering and so $y^{i+1} - y^i \geq 0$, hence $\{y^i\}$ are non-decreasing.

  More generally, we may order such that $\{x^i\}$ are non-decreasing, containing both strictly increasing and constant subsequences.
  Over strictly increasing subsequences of $x$, we have non-decreasing $y$ as before.
  By uniqueness of $z^i$ (points in the support), $y$ cannot also be equal for constant $x$ subsequences, so over these, $z^i$ can be ordered such that $y^i$ are strictly increasing.
  Finally, note that for $z$ in a preceding strictly increasing sequence, and $z'$ the next subsequence, $x' - x > 0$, hence all $y$ in the preceding subsequence are lower than those in the next.

  Hence the support of any optimal transport map (for discrete, euclidean gd, 1D) can be expressed through two non-decreasing subsequences over the \gls{gs}.
  Similarly, in the limit we require monotone $T_x, T_y$ in the continuous case.
\end{proof}

\begin{corollary} \label{trace_monotonicity_coro}
  If $\sigp[!]{μ}{i}$ and  $\sigp[!]{ν}{j}$ are ordered, then it follows from the monotonicity of the support trace that
  \begin{equation}
    p_{ij} p_{i' j'} > 0
    \quad \text{with} \quad 
    i \leq i' 
    \quad \Rightarrow \quad
    j \leq j' .
  \end{equation}
\end{corollary}

\begin{definition}[Slacks]
  \begin{cgather}
    s_μ^i = \sigw{μ}{i} - \sum_{ j=1 }^M p_{ij}
    \quad \quad \text{and} \quad \quad
    s_ν^j = \sigw{ν}{j} - \sum_{ i=1 }^N p_{ij} ,
  \end{cgather}where 
  $s_μ^i = 0$ signals the $i^\text{th}$ row constraint being met.
\end{definition}

\begin{definition}[North-west corner rule (NWCR)]
  Set the north-west corner of the transport plan to $p_{11} = \min (\sigw{μ}{1}, \sigw{ν}{1})$, so that at least one of $s_μ^1 = 0$ or $s_ν^1 = 0$ is satisfied.
\end{definition}

\begin{definition}[North-west corner algorithm (NWCA)]
  Having satisfied at least one row or column constraint, we move on to determine the remainder of the transport plan, which is structured
  \begin{equation}
    \mathbf{P} (μ, ν) = 
    \begin{pmatrix}
      p_{11} \quad \mathbf{0}\\
      \mathbf{P} (μ^1, ν^1)
    \end{pmatrix}
    \quad \text{or} \quad
    \mathbf{P} (μ, ν) = 
    \begin{pmatrix}
      \begin{matrix}
        p_{11} \\
        \mathbf{0}
      \end{matrix} &
      \mathbf{P} (μ^1, ν^1)
    \end{pmatrix}
  \end{equation}
  where $μ^1 = μ - p_{11} δ_{\sigp[!]{μ}{1}}$ and $ν^1 = ν - p_{11} δ_{\sigp[!]{ν}{1}}$.
  To determine $\mathbf{P} (μ^1, ν^1)$, and therefore $\mathbf{P} (μ, ν)$, we iterate this rule until $ μ^K (\mathcal{X}) = ν^K (\mathcal{X}) = 0 $.
  This yields a feasible transport plan (as all slacks are set to $0$) in at most $N+M-1$ steps\footnote{Too see why, consider that each step satisfies at least one additional constraint (than the previous). There are a total of $N+M$ constraints, and the final step must satisfy both the last row and column constraints. Alternatively, the longest monotonic trace through the matrix $\mathbf{P} (μ, ν)$ visits $N+M-1$ entries.} (see 3.4.2 in \cite{peyre2020computational}).
\end{definition}

\begin{algorithm}
  \caption{North-west corner algorithm} \label{monoFlowAlg}
  \begin{algorithmic}%
  
  \Function{NWCA}{$\protect\sigw{μ}{}$, $\protect\sigw{ν}{}$}
  
    \State $i_1 \gets 1, j_1 \gets 1, s_μ \gets \sigw{μ}{1}, s_ν \gets \sigw{ν}{1}, \mathbf{P} \gets \mathbf{0}_{N \times M}$, $\text{KC}=0$
  
    \For{$k=1$ to $k=N+M-1$}
  
    \State $p_{i_k j_k} \gets \min (s_μ, s_ν)$
  
    \State $s_μ \gets s_μ - p_{i_k j_k} $
  
    \State $s_ν \gets s_ν - p_{i_k j_k} $

    \State $\text{KC} = \text{KC} + c (\sigp[!]{μ}{i_k}, \sigp[!]{ν}{j_k}) p_{i_k j_k}$
  
      \If{$s_μ=0$}
  
          \State $ i_{k+1} \gets i_k + 1 $, $ j_{k+1} \gets j_{k} $
  
          \If{$i_{k+1} = N$}
  
              \State \textbf{break}
  
          \EndIf
  
          \State $ s_μ \gets \sigw{μ}{i_{k+1}} $
  
      \EndIf

      \If{$s_ν=0$}
  
        \State $ i_{k+1} \gets i_{k} $, $ j_{k+1} \gets j_k + 1 $
  
          \If{$j_{k+1} = M$}
  
              \State \textbf{break}
  
          \EndIf
  
          \State $ s_ν \gets \sigw{ν}{j_{k+1}} $
  
      \EndIf
  
    \EndFor

    \EndFunction
  
  \end{algorithmic}
  
\end{algorithm}

\begin{figure}[htbp]
  \centering
  \begin{subfigure}{.38\textwidth}
    \centering
    \includegraphics[width=.99\linewidth]{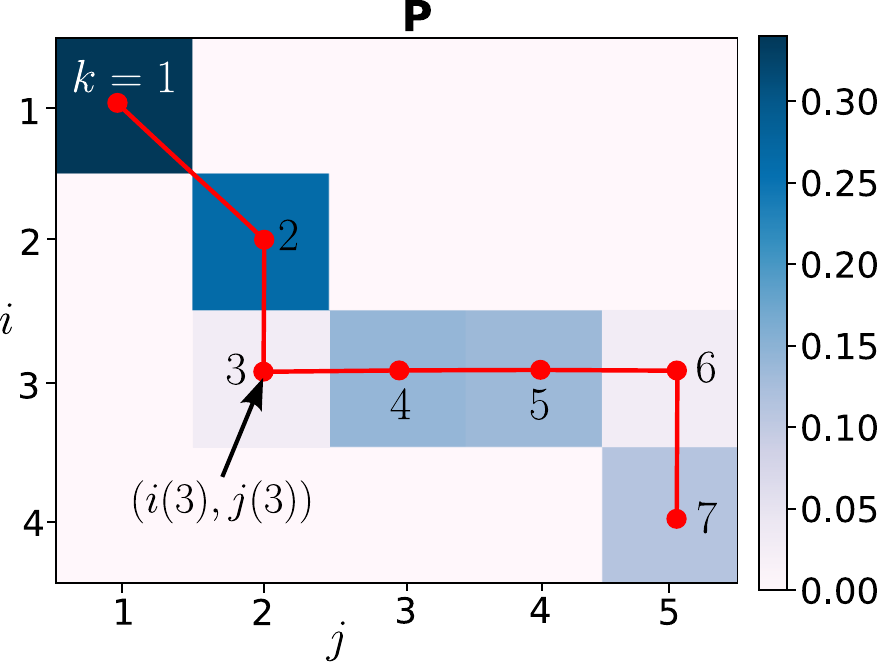}
    \caption{Optimal Kantorovich plan.}
  \end{subfigure}
  \begin{subfigure}{.6\textwidth}
    \centering
    \includegraphics[width=.99\linewidth]{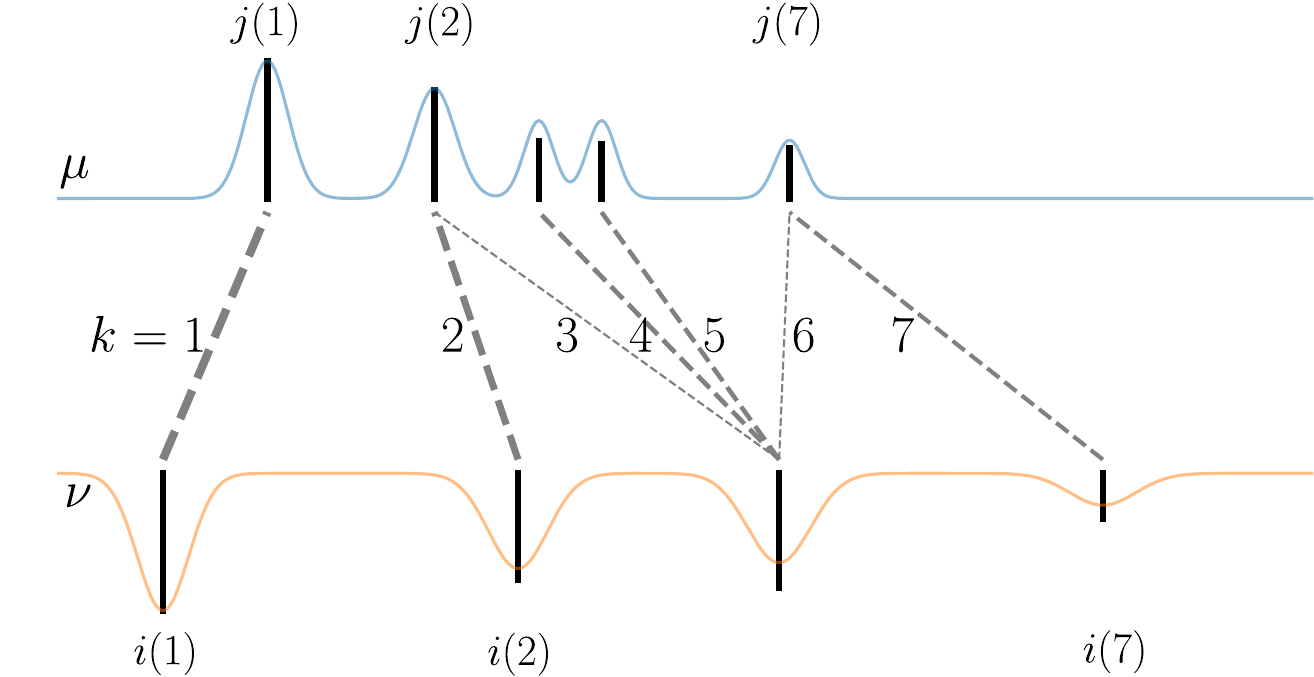}
    \caption{Edges corresponding to $ p^*_{ij}>0 $ do not cross.}
  \end{subfigure}%
  \caption
  {
    The North-west corner algorithm (\cref{monoFlowAlg}) and the monotonic trace produced, corresponding to monotonic parameterisations $i(k)$ and $j(k)$, or in the continuous case $x(t)$ and $y(t)$.
    As the name suggests, we start in the north-west corner of the transport plan, and at each step $k$, compare how much mass is left at each marginal index $i$ and $j$, transfer the most mutually available, 
    and increment the saturated index/indices ($i$, $j$, or both).
    Note, the horizontal stepping of the Kantorovich plan in the middle panels of \cref{displacement_interp_vis}, how a Monge function cannot do this, and how a function pair $(x(t), y(t))$ can (see \cref{trace_monotonicity}), as one can linger at the same position for some time, e.g. $i(k)=3$ for $k \in [3,6]$.
  }
  \end{figure}

\FloatBarrier
\subsubsection{Optimality for 1D sorted measures} \label{nwca_sec}

For $μ \in P (\mathbb{R})$, its signature may be sorted so that
\begin{equation} \label{support_ordering}
  \sigp[!]{μ}{1} <  \sigp[!]{μ}{2} \ldots <  \sigp[!]{μ}{N-1} <  \sigp[!]{μ}{N}
\end{equation}
noting that if $\sigp[!]{μ}{i} = \sigp[!]{μ}{i'}$, we can simplify $ \sigw{μ}{i} δ_{\sigp[!]{μ}{i}} + \sigw{μ}{j} δ_{\sigp[!]{μ}{i'}} = (\sigw{μ}{i} + \sigw{μ}{i'}) δ_{\sigp[!]{μ}{i}} $.

\begin{lemma} \label{edge_cond}
  Taking 1D $μ$ and $ν$ to be sorted without loss of generality,
  the NWCR is a necessary condition for a monotone plan.
\end{lemma}

\begin{proof}
  Supposing this is not true, if neither the mass constraint at $\sigp[!]{μ}{1}$ or $\sigp[!]{ν}{1}$ is satisfied by $p_{11}$ alone then
  \begin{equation}
    s_μ^1 > 0
    \: \: \Rightarrow \: \:
    \exists p_{i1}>0
    \quad \text{and} \quad  
    s_ν^1 >0
    \: \: \Rightarrow \: \:
    \exists p_{1j'} >0
  \end{equation}
  with $i,j'>1$,
  but by monotonicity (recall \cref{trace_monotonicity_coro}), since  $p_{1j'} p_{i1} >0$ (where $1<i$) we must have $ j ' \leq 1 $, which contradicts $ 1 < j'$.
\end{proof}

\begin{corollary} \label{NWCA_monotonicity}
  If $\mathbf{P} (μ,ν)$ is to be monotone, $\mathbf{P} (μ^k, ν^k)$ must be monotone for all $k$, hence the application of the \gls{nwca} is necessary for $\mathbf{P} (μ,ν)$ to be monotone.
\end{corollary}

\begin{theorem}
  In 1D, application of the \gls{nwca} on sorted signatures produces an optimal map, provided the \gls{gc} satisfies $c(x,y)+c(x',y') \leq c(x,y')+ c(x',y)$ for all $\langle x'-x,y'-y \rangle \geq 0$.
\end{theorem}

\begin{proof}
  As by Proposition 2.5 of \cite{monoOTrieger}, there is only one monotone increasing transport plan with marginals $μ,ν$ in 1D and as we require a monotonic increasing plan for optimality, we have the unique optimal solution.
\end{proof}

\subsection{Optimal distance approximation quantisers} \label{proxy_patterns}

Though the computation of a complex distance calculation might scale badly in the size of the original patterns compared, we may be able to economically approximate this via a much smaller but nearby quantised counterparts.
In other words, there may exist much smaller but similar patterns that are much faster to compare.
For instance the `signatures' used in \cite{Rubner00}, generated via k-means clustering from some original distribution (which in practice will also be discrete but larger).

It is well known that k-means clustering in a distribution is equivalent to looking for its $W_2$ optimal k point quantiser \cite{Lloyd82, Pollard82, canas2012learning} (see \cref{W2quant} for a quick justification).
For the \gls{ot} case this can be thought of as approximating the cost of delivering some goods between two sets of consumers, by the cost of the transport between major warehouses/hubs, where so long as the hubs are appropriately placed, this will represent the bulk of the overall transport (scale separation \cite{Shirdhonkar08,ggggg}).
\begin{figure}[htbp]
\center
\includegraphics[width=\textwidth]{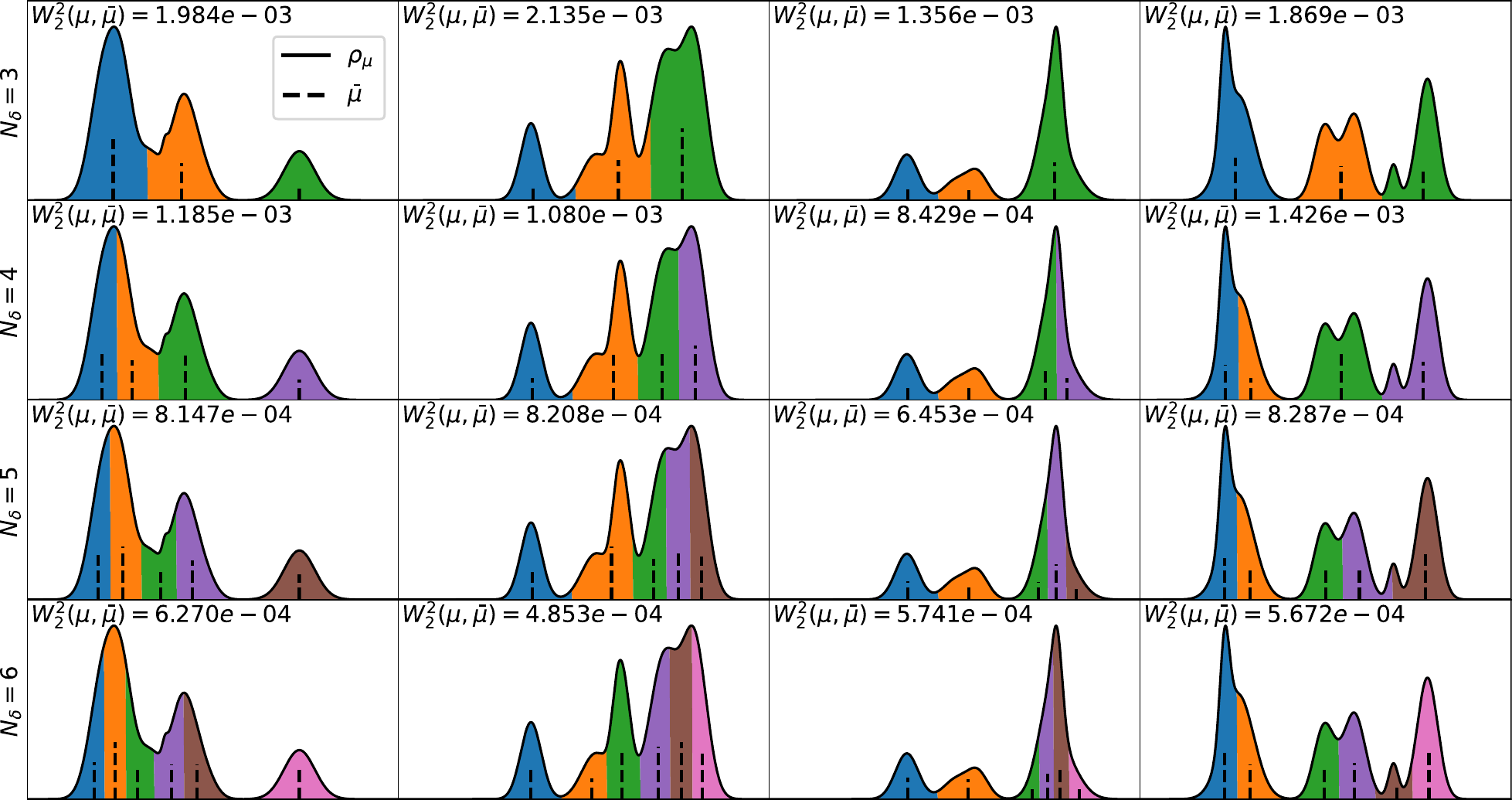}
\caption{
  Quantisation of various continuous measures $μ$ with density $ρ_μ$ (one per column), into a discrete measure $\bar{μ}$ of various numbers of Diracs $N_δ$ (one per row).
  The colored region surrounding a dashed stick denotes the continuous area that is condensed into that stick, with taller sticks (larger weight) having absorbed larger areas.
  With lower quantisation cost (higher $N_δ$), we of course have better proxies for estimating the Wasserstein (p=2) distance between different $μ$s.
}
\end{figure}

\begin{figure}[htbp]
\centering
\includegraphics[width=0.4\textwidth]{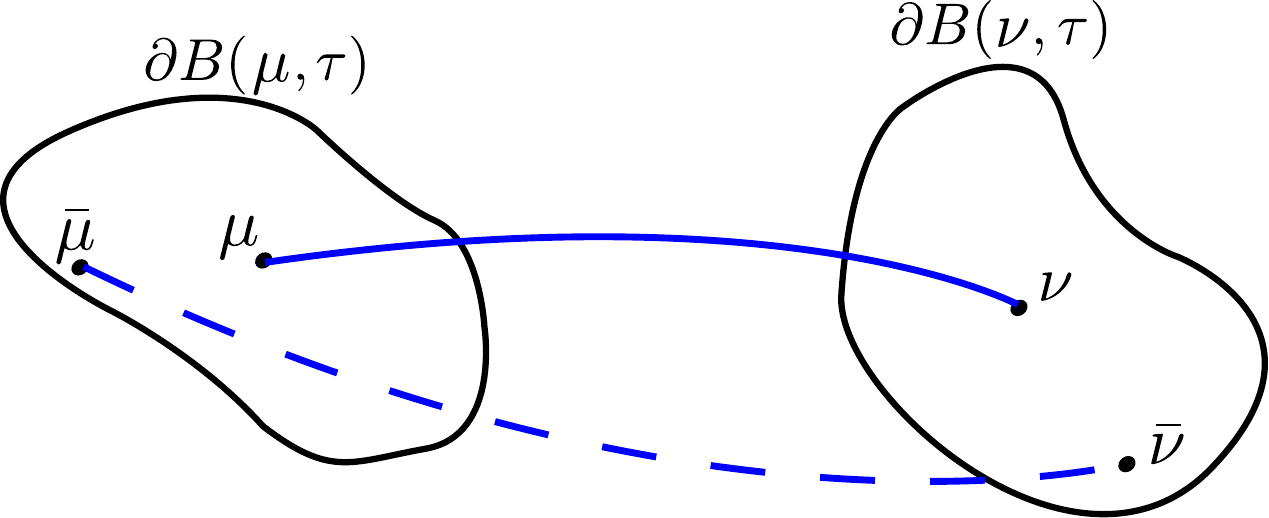}
\caption{
  Approximation of $d(μ, ν)$ through $d(\bar{μ}, \bar{ν})$ where
  $\bar{μ},\bar{ν}$ are within balls of radius $τ$ (solid black boundary), where
  $d(\bar{μ}, \bar{ν}) \rightarrow d(μ, ν)$ as $τ \rightarrow 0$ by the triangle inequality.
  For a geodesic distance, the distance $d(μ, ν)$ is the length of the geodesic connecting $μ$ and $ν$ (solid blue line, while dashed connects $\bar{μ}$ and $\bar{ν}$).
}
\end{figure}

Let the closest quantised pattern of at most size $N$ be
\begin{cgather}
\bar{μ}_N = \argmin_{m \in \mathcal{P}} d_\mathcal{P} (μ, m) \\
 \text{size} (m) \leq N,
\end{cgather}where 
for a space of measures we can take $\text{size} (m) = \text{supp} (m)$.
So long as $ \lim_{N \to \infty} d_\mathcal{P} (\bar{μ}_N, μ) = 0 $, then $ \exists N'$ for which
\begin{equation} \label{proxy_triangle_ineq}
\vert d_\mathcal{P} (μ, ν) - d_\mathcal{P} (\bar{μ}_{N'}, \bar{ν}_{N'}) \vert \leq d_\mathcal{P} (\bar{μ}_{N'}, μ) + d_\mathcal{P} (\bar{ν}_{N'}, ν) < 2 τ,
\end{equation}
from the triangle inequality.
Additionally, with monotonically improving proxies $ d_\mathcal{P} (\bar{μ}_{N+1}, μ) < d_\mathcal{P} (\bar{μ}_N, μ) $, the distance approximations also improve monotonically.
Appendix \cref{W2quantmono} shows that these conditions do indeed hold for $W_2$.

\begin{figure}[htbp]
  \centering
  \includegraphics[width=0.7\textwidth]{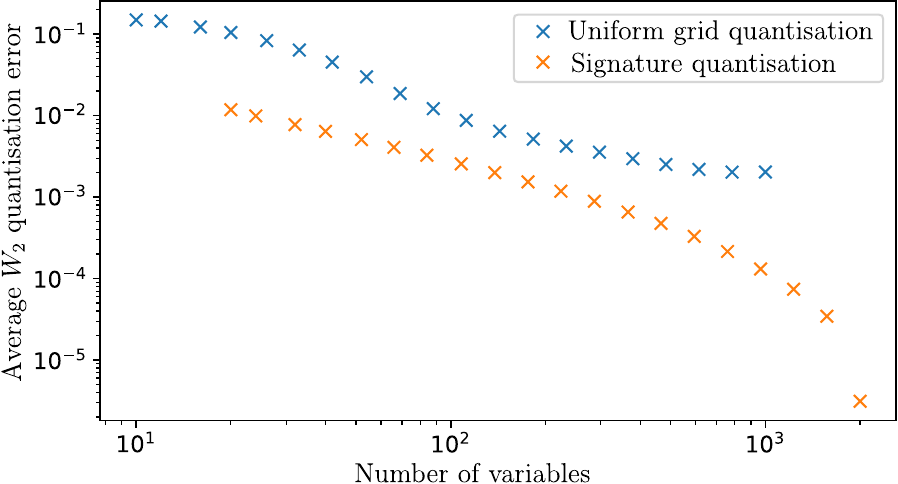}
  \caption{
    Average error (evaluated over 1000 random \gls{gmm} pattern pairs) vs variables required for two quantisation choices.  
    \glspl{gmm} are generated as in \cref{GMM_dataset_sec}, their density is initially sampled on a uniform grid of 1024 points.
    Signatures require twice the amount of variables for each grid size as we are specifying positions as well as weights (plot is adjusted accordingly).
    The signature quantisation is achieved through hierarchical agglomerative centroid-linkage clustering.
  } \label{fig:sig_quant}
\end{figure}

When using this approximation along with our unmixing framework, a parallel may be noted with multilevel clustering \cite{ho2017multilevel}, as we are quantising/clustering both between and within our patterns.
Though of course in our case the varied basis distribution must also be learned and our discrete measures have weighted atoms.

In practice, hierarchical agglomerative centroid-linkage provides good proxies, while allowing a quantiser to be built iteratively.
Intuitively, similarity to k-means can be understood from noting that repeatedly merging Diracs that cause the smallest change in $W_2$ is likely to merge measures from the same Voronoi partition/cluster, when all points from a cluster are merged, we obtain a hierarchical approximation of the $W_2$ average of that cluster (see \cref{W2quant}).
The merging continues until a criterion is met, such as target size or a target distance error (resulting in signatures of various sizes per input measure).
\cref{fig:sig_quant} suggests that one can represent a distribution more accurately even with half as many ``sticks'' (Diracs) when compared to a uniform grid downsampling,
such that there is better value in controlling positions as well as weights (compared to just a greater number of weights).
In the context of the anomaly detection linear program, marginals of half the size results in transport plans of a quarter the size, which may result in solve time speed up factors of $16 - 64$ based on $\mathcal{O}(n^2) - \mathcal{O}(n^3)$ scaling for solving a naive linear program (while using more accurate marginals).

\FloatBarrier
\section{Experimental evaluation} \label{results_sec}

To demonstrate the flexibility of our method, we performed three experiments,
each representing a different mixing context.
In \cref{GMM_dataset_sec}, we explore blind unmixing performance across $1200$ randomly generated large datasets, featuring various levels of basis variation.
The independent variation of each peak represents a more extreme version of what might otherwise occur in \gls{pxrd} patterns.
In \cref{PXRD_data_sec}, we explore non-blind unmixing within a dataset corresponding to a more realistic, physically informed, simulation of \gls{pxrd} patterns. 
This features a large basis, with some basis patterns being hard to distinguish after variation.
In \cref{hyperspectral_data}, we explore blind unmixing across three smaller laboratory hyperspectral imaging datasets corresponding to different physical mixing scenarios, each with different resulting mixing models.

We compare against two other general\footnote{General in that they do not make any assumptions corresponding to a particular discipline or kind of spectra.} unmixing methods.
These are
\gls{nmf} and \gls{onmf} (with Kullbach-Leibler and Frobenius divergence based costs respectively).
During blind unmixing, each method is tested for a constant $N_\text{iter}=60$.
In our visualisations, our derived signatures will be represented via kernel density estimation\footnote{If our signatures are taken to approximate continuous objects, this process approximates them by spreading out mass that was previously concentrated during discretisation.}, using a Gaussian kernel of specified width.

As it is standard to provide \gls{rmse} between ground truth and inferred quantities to measure performance, 
we will do so as well, wherever possible, however, as we are working with signatures, in order to measure these we must first approximate densities and evaluate them on a common grid.
In such a case $\mathbf{B}$ refers to this further approximation, while $\up{B}$ is a tuple of measures as usual.

\begin{equation}
  \text{RMSE} (\mathbf{B}, \mathbf{B}') = \sqrt{ \frac{1}{N M} \sum_{i=1}^{N} \sum_{j=1}^{M} (b_{ij} - b_{ij}')^2 }
  \label{eq:rmse}
\end{equation}

To facilitate a more direct comparison, we will define a dissimilarity based a shift invariant modification of $W_2$, which we will refer to as ``centered Wasserstein''
\begin{cgather}
  \overline{W}_2 (μ, ν) = W_2 (T_\#^{-μ_b} μ , T_\#^{-ν_b} ν ) \\
  T^{Δ x}(x) = x + Δx
  \quad \text{with} \quad
  μ_b = \int x d μ (x) .
\end{cgather}

When ground truth is available we will use the following metrics to gauge performance, ground truth quantities will be decorated with a ``t'' subscript. 
\begin{enumerate}

  \item Inter Basis Centered Wasserstein (IBCW)
  \begin{equation}
  \text{IBCW} ( \up{B}, \up{B}' ) = \frac{1}{K} \sum_{i=1}^K \overline{W}_2 (b^i, b'^i)
  \end{equation}
  We can similarly define inter component and basis-component variants, though these seemed to follow similar trends in our experiments.

  \item Average absolute difference in percentage abundances (AADPA).
  \begin{equation}
  \frac{100}{K M} \sum_{i=1}^K \sum_{j=1}^M \vert a_i^j - {(a_t)}_i^j \vert
  \end{equation}
  Though the \gls{rmse} between abundances is perfectly appropriate, this is offered as a more interpretable alternative.

\end{enumerate}

As a standard \gls{onmf} implementation can only shift basis patterns to the right, the learned basis patterns will be offset, we center these in our plots and for \gls{rmse} calculations.
Also note that with discrete convolutions we can only broaden and not thin patterns, so the basis patterns learned correspond to thinnest, rather than average variations.

\subsection{Random GMM datasets} \label{GMM_dataset_sec}

Suppose our patterns are 1D \gls{gmm}'s. 
We express our random measures with respect to the Gaussian density $ g (x ; μ, σ) $ for mean $μ$ and standard deviation $σ$
\begin{equation}
  ρ_{b^i} (x)  = \sum_{k=1}^{N_G} I^i_k g ( x ; {μ^i_k, σ^i_k} ) ,
\end{equation}
where a continuous measure $ν$ can be related to its density via $ν (S) = \int_S ρ_{ν} (x) dx$ for $S \subset \mathbb{R}$.
This distribution is itself randomly drawn by randomly selecting 
\begin{equation}
I^i_k \sim \vert \mathcal{N} (1, 0.25) \vert ,
\quad
μ^i_k \sim U (0.2, 0.8)
\quad \text{and} \quad
σ^i_k \sim 0.02 \cdot 2^{\mathcal{N} (0, 0.5)} .
\end{equation}
Peak intensity is denoted by $I$ (where we must normalise $ \sum_{k=1}^{N_G} I_k^i = 1 $), location by $μ$ and breadth by $σ$.
Variation is then induced by subsequently perturbing the parameters of each peak $k$ away from those of the basis patterns, independently, per each warped basis $j$, per each mixture $i$
\begin{equation}
  ρ_{\widetilde{b}^{ij}} (x) = \sum_{k=0}^{N_G} I^i_k g ( x ; {μ^{ij}_k, σ^{ij}_k} ) 
\end{equation}
where
\begin{equation}
μ^{ij}_k = μ^{i}_k + Δ μ^{ij}_k ,
\quad
σ^{ij}_k = σ^i_k \cdot γ^{ij}_k .
\end{equation}
with
\begin{equation}
Δ μ^{ij}_k \sim \mathcal{N} (0, σ_{Δ μ})%
\quad \text{and} \quad
γ^{ij}_k \sim 2^{\mathcal{N} (0, σ_p)} . %
\end{equation}

The exponential is used as we would like a broadening such as $σ=σ\times 2$ to be roughly as likely as thinning by the same factor $σ=σ/2$.
Mixed examples can then be created by combining $ c^j = \sum_{i=1}^{K} h^j_i \widetilde{b}^{ij} $ where $ \mathbf{h}^j $ is drawn from a uniform distribution on $ S^{K-1} $. 
We perform our assessment across three different basis variation strengths summarised in~\cref{basis_variation_GMM}.

\begin{table}[htbp]
    \begin{tabular}{ | c | c | c | c | c |}

    \hline
    Preset & $ σ_{Δ μ} $ & $ σ_p $ & 95\% of $Δ μ^{ij}_k$ are in & 95\% of $σ^{ij}_k$ are in  \\
    \hline
    Low & 0.015 & 0.15 & $[-0.03, +0.03]$   & $~ [0.81 σ^i_k, 1.23 σ^i_k]$ \\
    Medium & 0.025 & 0.3 & $[-0.05, +0.05]$ & $~ [0.66 σ^i_k, 1.52 σ^i_k]$ \\
    High & 0.035 & 0.5 & $[-0.07, +0.07]$   & $~ [0.50 σ^i_k, 2.00 σ^i_k]$ \\
\hline
\end{tabular} 
\caption{Basis variation strength presets used, where two variations of the same peak can be found as far as 14\% of the domain width apart.}\label{basis_variation_GMM}
\end{table}

\begin{figure}[htbp]
  \begin{subfigure}{0.99\textwidth}
    \centering
    \includegraphics[width=.99\linewidth]{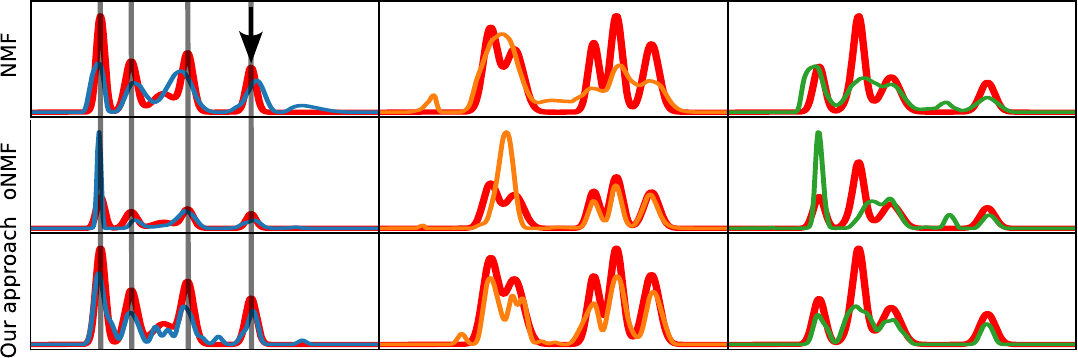}
    \caption{
      The basis attained with each decomposition method tested, one method per row.
      Each column corresponds to one basis pattern, which also has an associated color (e.g. $b^1$ in blue),
      while the ground truth basis patterns are overlayed in red.
    }
  \end{subfigure}

  \begin{subfigure}{0.99\textwidth}
    \centering
    \includegraphics[width=.99\linewidth]{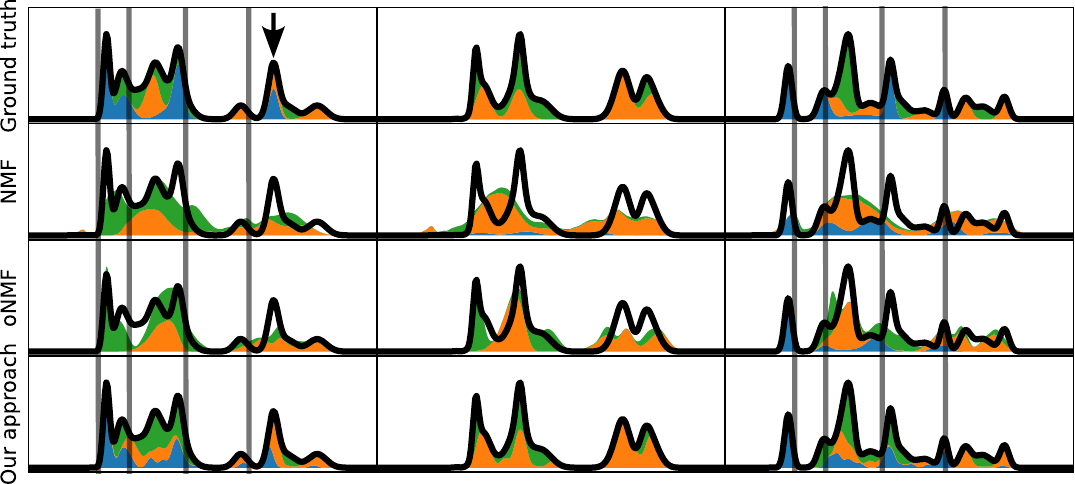}
    \caption{
      Mixed datapoints $c^j$ (black line), one per column, as a sum of transformed bases $\widetilde{b}^{ij}$ (area with coloring corresponding to $b^i$), according to ground truth, and the tested methods, one per row.
    }
  \end{subfigure}
\caption{
  Visual comparison of decompositions of one randomly generated GMM dataset (high variation preset), according to each tested method.
  Signatures are visualised by smoothing with a Gaussian kernel with $\sigma_\text{KDE} = 0.004$.
  Note that the mixture in the first column of (b), has a component of the first basis (blue) in which some peaks are found further to the left and others are found further to the right, when compared to the left column of (a).
  Also, note that the mixture in the third column of (b), has a component with a different set of non-uniform shifts.
  Therefore, for a decomposition method to be able to explain both of these mixtures with respect to the same basis, 
  it must have the ability to perturb each peak independently.
} \label{GMM_example}
\end{figure}

\begin{table}[hbt]
    \centering
    \small
    \begin{tabular}{ | c | c | c | c | c | c | }

    \hline
    \multicolumn{2}{| c |}{} & \multicolumn{4}{c |}{Metric} \\
    \hline
    Variability & Approach & AADPA$ ( \mathbf{A}, \mathbf{A}_t ) $ & RMSE$ ( \mathbf{A}, \mathbf{A}_t ) $ & RMSE$ ( \mathbf{B}, \mathbf{B}_{t} ) $ & IBCW$( \up{B},  \up{B}_t) $ \\
    & & & $\times 100$ & $\times 1000$ & $\times 100$ \\
    \hline 

    & \gls{nmf} & $5.783 \pm 2.469$ & $7.909 \pm 3.421$ & $1.742 \pm 0.569$ & $\mathbf{1.294} \pm 0.602$ \\
Low & \gls{onmf} & $8.300 \pm 2.921$ & $10.882 \pm 3.722$ & $4.611 \pm 1.454$ & $3.559 \pm 1.507$ \\
    & Our approach & $\mathbf{4.274} \pm 1.973$ & $\mathbf{5.544} \pm 2.757$ & $\mathbf{1.668} \pm 0.485$ & $1.756 \pm 0.894$ \\
\hline
       & \gls{nmf} & $10.247 \pm 3.800$ & $14.142 \pm 5.120$ & $3.130 \pm 0.926$ & $2.262 \pm 1.056$ \\
Medium & \gls{onmf} & $11.955 \pm 3.134$ & $15.539 \pm 3.846$ & $7.138 \pm 1.710$ & $5.220 \pm 1.771$ \\
       & Our approach & $\mathbf{6.019} \pm 2.064$ & $\mathbf{8.154} \pm 2.932$ & $\mathbf{2.574} \pm 0.587$ & $\mathbf{1.837} \pm 0.853$ \\
\hline
     & \gls{nmf} & $14.167 \pm 4.023$ & $19.276 \pm 5.180$ & $4.271 \pm 1.010$ & $3.299 \pm 1.243$ \\
High & \gls{onmf} & $14.413 \pm 2.928$ & $18.743 \pm 3.550$ & $8.723 \pm 1.636$ & $5.769 \pm 1.589$ \\
     & Our approach & $\mathbf{8.371} \pm 2.493$ & $\mathbf{11.509} \pm 3.451$ & $\mathbf{3.625} \pm 0.732$ & $\mathbf{2.154} \pm 0.800$ \\

\hline
\end{tabular}
\caption{Mean decomposition metrics over 400 random \gls{gmm} datasets per variability preset.
         These were generated with $ K=3 $ and $M=800$, where the density of each $c^j$ is initially sampled uniformly with $N=256$ then discretised with $N_δ=64$.
         The best of eight decompositions with different initialisation (as measured by final cost), is taken for each dataset/method.
         The ground truth basis ($\up{B}_t$) is also similarly sampled (producing $\mathbf{B}_{t}$), as is the kde density of the inferred basis signature for basis \glspl{rmse}.
         The \gls{onmf} is performed with a maximum shift of 25 gridpoints and the inferred basis patterns are centered by shifting for 13 gridpoints before measuring RMSE$ ( \mathbf{B}, \mathbf{B}_{t} ) $.
        } \label{GMM_trials_stats}
\end{table}

We can see from \cref{GMM_trials_stats} that sometimes the \gls{onmf} model can perform worse than \gls{nmf} in extracting the information we are after, this is despite the fact that our \gls{onmf} solutions are initialised with \gls{nmf} solutions.
Even if the overall fit can improve, 
the accuracy of inferred constituents can simultaneously suffer (see \cref{GMM_example}).
In fact, the more flexible a composition model, the more trivially a profile can be fit, regardless of basis. 
For instance, if one were to use \gls{ot} as a basis variation model $ϕ(T;b) = T_\# b$, 
one could arbitrarily warp any given basis to perfectly fit any mixture,
not giving us any indication as to weather one basis is more appropriate over another.
This supports our original reasoning for defining a model driven by a non-reconstruction cost.

\FloatBarrier
\subsection{PXRD unmixing} \label{PXRD_data_sec}

As an indicator of suitability to \gls{pxrd} patterns, 
we can generate a benchmarking dataset with the Synthetic instance generator \cite{synthetic_instance_gen}, which provides a more physically informed model of such data (compared to \glspl{gmm}).
Similarly to the previous, this too features non-uniform shifting and broadening.
The default preset of this generator (the AlLiFe system) features a relatively large basis (6 patterns, corresponding to pure phases).
As we found that none of the methods we compare here could \textit{consistently} blindly discover all 6 of the basis patterns, 
we assess separation/abundance estimation in the non-blind (see \cref{variants_table}) regime here, the performance in this task is summarised in \cref{pxrd_table} and illustrated in \cref{PXRD_preview}.

\begin{figure}[htbp]
  \begin{subfigure}{0.99\textwidth}
    \centering
    \includegraphics[width=.99\linewidth]{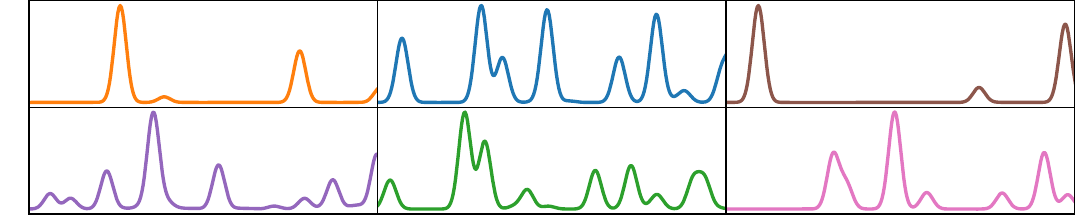}
    \caption{
      The 6 ground truth basis patterns.
    } \label{PXRD_preview_a}
  \end{subfigure}

  \begin{subfigure}{0.99\textwidth}
    \centering
    \includegraphics[width=.99\linewidth]{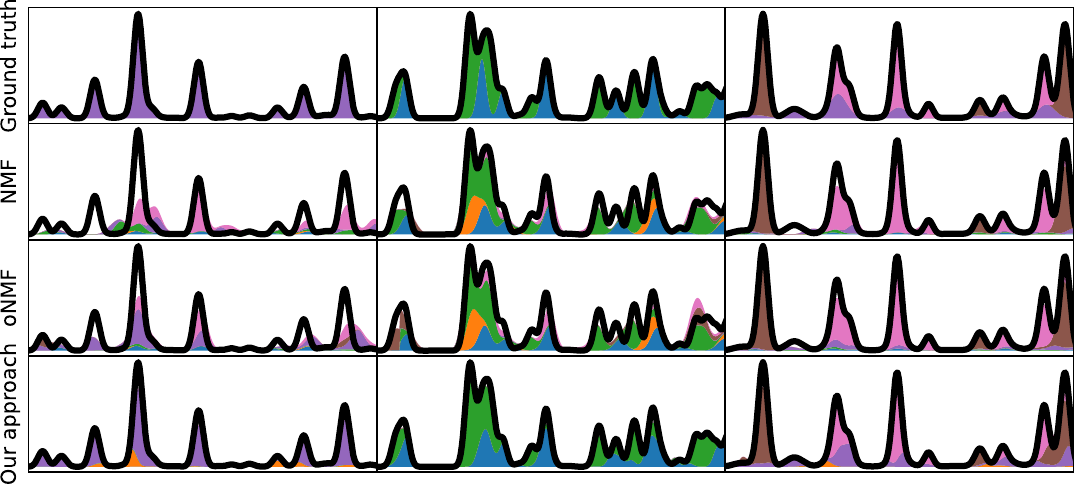}
    \caption{
      Mixed datapoints $c^j$ (black line), one per column, as a sum of transformed bases $\widetilde{b}^{ij}$ (area with coloring corresponding to $b^i$), according to ground truth, and the tested methods, one per row.
    } \label{PXRD_preview_b}
  \end{subfigure}
    \caption{
      Visual comparison of basis separation (non-blind unmixing) according to the tested methods (one per row), 
      of a Synthetic instance generator dataset with $K=6$, $m=M=1275$, $N_δ = 64$. 
      To visualise signatures, $\sigma_\text{KDE} = 0.188$ has been used, $x$ domain ($2θ$) has been truncated to $[16, 45]$ in these plots.
   } \label{PXRD_preview}
\end{figure}

\begin{table}[hbt]
  \begin{tabular}{ | c | c | c | }

  \hline
  \multicolumn{1}{| c |}{} & \multicolumn{2}{c |}{Metric} \\
  \hline
Approach & AADPA$ ( \mathbf{A}, \mathbf{A}_t ) $ & RMSE$ ( \mathbf{A}, \mathbf{A}_t ) $ \\

  \hline 

  \gls{nmf}    & 10.467 & 0.203  \\
  \gls{onmf}   &  6.825 & 0.118 \\
  Our approach &  \textbf{3.115} & \textbf{0.055} \\

\hline
\end{tabular}
\caption{Abundance derivation metrics for the compared approaches (in the separation-only/fixed-basis regime), over the synthetic instance generator dataset.} \label{pxrd_table}
\end{table}

\cref{pxrd_table} shows significantly lower error in abundance prediction for our approach.
\cref{PXRD_preview} again demonstrates the merit of a non-reconstruction based cost,
in that, though the sums of the components come close to matching the profile of the mixture for \gls{nmf} and \gls{onmf} in the middle column of \cref{PXRD_preview_b}, 
they are not able to correctly relate features of the mixture with features from the basis patterns, 
and therefore wrongly partition the area under a mixture, 
and derive less accurate varied bases and abundances.
In the right column of \cref{PXRD_preview_b}, 
where the purple pattern cannot be adequately varied by \gls{nmf} and \gls{onmf},
they over assign mass to the similar, pink pattern instead.

\FloatBarrier
\subsection{Hyperspectral imaging} \label{hyperspectral_data}

This dataset \cite{zhao2019laboratorycreated} contains 3 scenes containing hyperspectral images of colored powders prepared in a laboratory in known amounts to showcase some different mixing scenarios in \gls{hi}, these are as follows.

\textbf{Scene 1 (Linear)}, 3 distinct powders are arranged in a grid of squares smaller than that of the hyperspectral instruments pixels, the subpixel pattern is repeated periodically.
A light source points directly at the powders, light generally travels from source, to a single powder (without significant scattering), to a hyperspectral pixel, though rays passing by each of the subpixel squares can make it to the same pixel, adding to the measurement independently.

\textbf{Scene 2 (Intimate)}, 4 distinct powders are mixed uniformly, the same ray of light can scatter off of multiple differently colored grains before reaching the sensor, as such this contains second-order reflections (and possibly higher) involving every combination of basis patterns.

\textbf{Scene 3 (Limited bilinear)}, 3 distinct powders are mixed as in the first case, though now the light source points at an adjacent plastic board, such that most rays will go from source, to board, to a subpixel square, to a pixel. This contains second-order reflections of the board with each of the first three basis patterns.

\begin{figure}[htbp]
  \begin{subfigure}{0.99\textwidth}
    \centering
  \includegraphics[width=.99\linewidth]{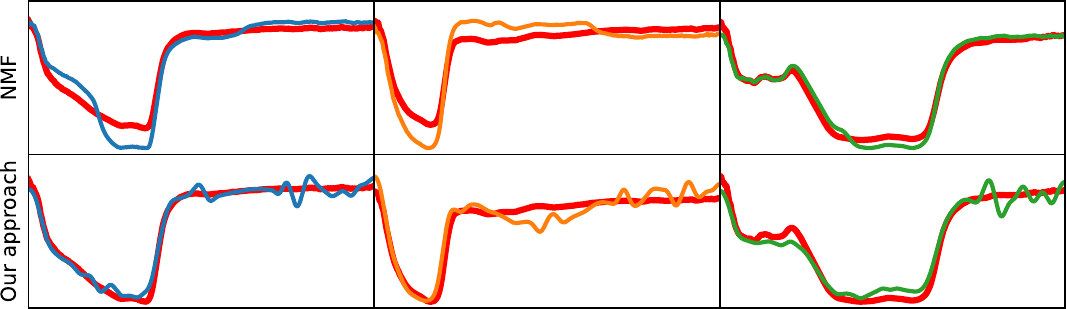}
    \caption{
      The basis attained with \gls{nmf} vs our method.
      Each column corresponds to one basis pattern, which also has an associated color (e.g. $b^1$ in blue),
      while the ground truth basis patterns are overlayed in red.
    }
  \end{subfigure}

  \begin{subfigure}{0.99\textwidth}
    \centering
  \includegraphics[width=.99\linewidth]{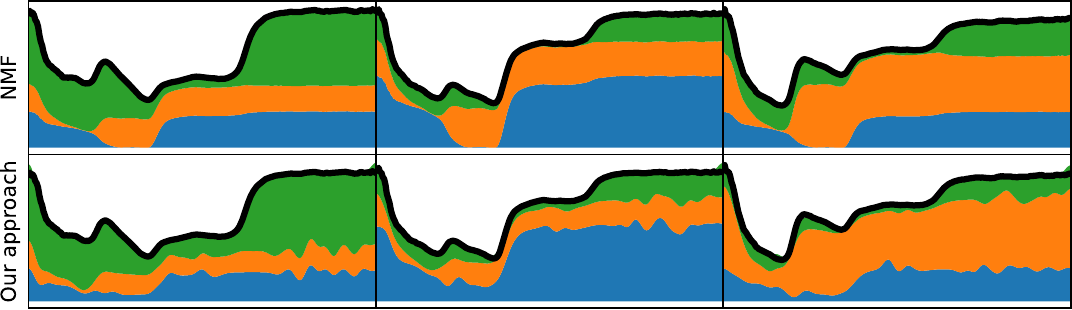}
    \caption{
      Mixed datapoints $c^j$ (black line), one per column, as a sum of transformed bases $\widetilde{b}^{ij}$ (area with coloring corresponding to $b^i$), according to \gls{nmf} in the top row and our approach in the bottom row.
    }
  \end{subfigure}
\caption{
  Visual comparison of decompositions of scene 1.
  To visualise signatures, $\sigma_\text{KDE} = 7.53$ has been used.
} \label{hi_scene2_vis}
\end{figure}

\begin{table}[hbt]
    \centering
    \small
    \begin{tabular}{ | c | c | c | c | c | c | }

    \hline
    \multicolumn{2}{| c |}{} & \multicolumn{4}{c |}{Metric} \\
    \hline
    Scene & Approach & AADPA$ ( \mathbf{A}, \mathbf{A}_t ) $ & RMSE$ ( \mathbf{A}, \mathbf{A}_t ) $ & RMSE$ ( \mathbf{B}, \mathbf{B}_{t} ) $ & IBCW$( \up{B},  \up{B}_t) $ \\
    & & & $\times 10$ & $\times 1000$ &  \\
    \hline 

    1 & \gls{nmf} & $4.561$ & $\mathbf{0.549}$ & $2.072$ & $45.269$ \\
      & Our approach & $\mathbf{4.437}$ & $0.595$ & $\mathbf{1.947}$ & $\mathbf{41.683}$ \\
\hline
    2 & \gls{nmf} & $11.697$ & $1.486$ & $1.830$ & $\mathbf{35.212}$ \\
      & Our approach & $\mathbf{9.787}$ & $\mathbf{1.384}$ & $\mathbf{1.648}$ & $40.198$ \\
\hline
    3 & \gls{nmf} & $12.788$ & $1.602$ & $2.458$ & $49.157$ \\
      & Our approach & $\mathbf{3.065}$ & $\mathbf{0.459}$ & $\mathbf{1.985}$ & $\mathbf{40.403}$ \\

\hline
\end{tabular}
\caption{
Performance summary of \gls{nmf} vs our method (in the fully blind regime), on a \gls{hi} dataset.
The best of 16 decompositions with different initial starts (as measured by final cost), is taken for each scene/method.
\gls{onmf} has been omitted in this instance due to instability.
} \label{hi_perf_table}
\end{table}

Each scene contains a number of images (10, 14 and 8 respectively), 
where each image corresponds to different mixing ratios.
While each pixel in each image could contain unique basis variation, due to the previously mentioned redundancy, we treat the Euclidean average\footnote{The Euclidean average is used simply to avoid resampling for NMF, which requires the inputs to be on the same grid.} pixel of each mixed image as our dataset.
Each pixel is a vector with 256 entries representing intensity at each point in the linearly spaced spectral band from $400$ to $1000$nm, we generate signatures with $N_\delta = 128$ for our method.
For each scene we perform fully blind unmixing.

Though a reconstruction based method might have otherwise used any additional degrees of freedom in order to try to provide a closer fit, injecting unnecessary variation (at the expense of ground truth recovery), scene 1 shows (see \cref{hi_perf_table}) that our method does not suffer from this issue (as we minimise basis variation).
Thus we provide similar decompositions in purely linear cases, as the appropriately constrained model, while having flexibility should more complex cases show up.

In cases featuring bilinear contributions $ κ_{\mathbf{A}^j}^\text{bin} \up{B} = \sum_{p,q} a_{pq}^j {b}^p \odot {b}^q $, 
we can still make progress with a linear combination function,
noting that our method has the flexibility to implicitly construct basis variations $\widetilde{b}^{ij}$ that are themselves some combination of the first-order and the second-order reflections of $b^i$ (with each of the other basis patterns).
The more second-order contributions in $c^j$, the more $\widetilde{b}^{ij}$ deviates from $b^i$.
Our approach somewhat improves on a linear decomposition on scene 2, but improves greatly on a linear decomposition of scene 3 where the presence of second-order reflections is more limited, but still seemingly significant.

\FloatBarrier
\section{Conclusion} \label{conclusion}

Many scientific disciplines work with ``mixed data,'' 
where each measurement results from a combination of underlying basis patterns that can themselves vary between samples (see \cref{basis_var_fig}).
Prevailing unmixing methods largely derive from a specific choice of variation function,
meaning that datasets that feature different forms of variation likely need to be processed with different specific methods. 
If no such method is available, significant adaptation is then required.
These methods typically optimize for reconstruction accuracy, that is, how closely the model’s output matches the measured signal. 
Reconstruction accuracy alone does not provide sufficient guidance, 
as multiple combinations of inputs might achieve a perfect fit for some mixtures.
Not all inputs are equally plausible, yet reconstruction does not distinguish between them.
Lastly, pointwise distances are often used in constructing a cost, which can vary erratically with respect to non-pointwise basis variation (see \cref{shift_convexity_fig}).

Our work introduces a different paradigm for unmixing based on geometric approaches (see \cref{sec:geom_methods}). 
One can imagine a set of seed inputs which transform into a latent point cloud after passing through the variation function $ϕ$, 
and into the final measurements after passing through the combination function $κ$.
Geometric approaches typically draw conclusions from geometric features of the point cloud of measurements, 
which can have a particular structure as a result of the mixing model.
We take this further by considering all latent point clouds that exactly produce the given measurements,
and weigh these by considering the kinds of point cloud that the variation function $ϕ$ can feasibly produce.
We show that for any $ϕ$ that is Lipchitz with respect to some distance $d_\mathcal{P}$, one can expect clustering in the metric space $(d_\mathcal{P}, \mathcal{P})$, 
where a smaller bound on maximum variation parameters (e.g. shifting/broadening) produces tighter clusters (see \cref{separability_sec}).
As such, we propose a (in-class) Fréchet variance based cost to search for a latent point cloud that correspond to the least basis variation possible.

To handle common non-pointwise basis variations like shifting and broadening, a corresponding non-pointwise distance metric is required.
The Wasserstein distance is a natural choice when signals are treated as distributions.
Our practical implementation leverages several key features of \gls{ot} to ensure efficiency:
quantization of distributions into sparse, representative ``signatures'';
a fast algorithm (\gls{nwca}) to compute exact 1D Wasserstein distances and transport maps in linear time;
and displacement interpolation for the efficient calculation of Wasserstein centroids (averages).
These innovations make our method computationally feasible for exploratory data analysis. 
As demonstrated through extensive experiments on random \gls{gmm}, simulated \gls{pxrd}, and laboratory \gls{hi} data,
our approach consistently outperforms other general-purpose methods in a range of unmixing scenarios.

This work provides a strong foundation for a wide range of applications and opens several avenues for future research. 
These include clustering with a full Fisher discriminant by including between-class variance terms,
exploring application to higher-dimensional data (this requires fast nD transport map calculation),
bidirectional partial matching during anomaly detection,
and, finally, exploring other ground distances, including learned distances, or even non-\gls{ot} distances all together.

\section*{Declarations}
These results were obtained as part of a PhD funded by Johnson Matthey.  The authors acknowledge support from the Leverhulme Trust via the Leverhulme Research Centre for Functional Materials Design (RC-2015-036),  UK Engineering and Physical Sciences Research Council (EPSRC) through grant number EP/V026887 as well as through AI for Chemistry: AIchemy Hub (EPSRC grant EP/Y028775/1 and EP/Y028759/1).  V. V. Gusev is supported by the Leverhulme Research Fellowship.

\begin{appendices}

  \section{Clustering and separability of metric space perturbations} \label{separability_sec}

  \cref{separability} gives us that, so long as
  $ϕ(λ; b)$ is Lipchitz with respect to $d_\mathcal{P}$, 
  it is possible to guarantee well defined basis variant clusters for any basis $\up{B} \in \mathcal{P}^K$
  (according to the criteria of \cref{voronoi_sep}), up to some (basis dependent) bound on the parameter $λ$.

\begin{lemma}[Voronoi separable] \label{voronoi_sep}
  With a maximum radius among clusters $ \max_{ij} d_\mathcal{P}(b^i, \widetilde{b}^{ij}) < R $ and adequate basis pattern spacing $ \min_{i \neq j} d_\mathcal{P}(b^i, b^j) > 2 R $ (or maximum similarity), the corresponding basis of some varied pattern can be easily identified, from the Voronoi diagram of the bases
  \begin{displaymath}
    \widetilde{b}^{jk} \in X^i
    \quad \Rightarrow \quad j=i
  \end{displaymath}
  for Voronoi cell $X^i$ (see \cref{voronoi_region}), and in fact
  \begin{displaymath}
    d_\mathcal{P}(b^i, \widetilde{b}^{jk}) < R
    \quad \Rightarrow \quad j=i ,
  \end{displaymath}
  where each cluster is enclosed in a ball which is enclosed in its own Voronoi cell, $ b^{i*} \subset B_R (b^i) \subset X^i $.
  
\end{lemma}

\begin{proof}
  We require that no variations spill into a neighbouring Voronoi cell, i.e. no two clusters touch/share a common point
  \begin{equation}
    \min_{\substack{ \forall i \neq k \\ \forall j,l }}
    d_\mathcal{P}(\widetilde{b}^{ij}, \widetilde{b}^{kl}) > 0 .
  \end{equation}
  From the triangle inequality 
  \begin{cgather}
    d_\mathcal{P}(b^i, b^k) \leq
    d_\mathcal{P}(b^{i}, \widetilde{b}^{ij}) + d_\mathcal{P}(\widetilde{b}^{ij}, \widetilde{b}^{kl}) + d_\mathcal{P}(\widetilde{b}^{kl}, b^k) \\
    \Rightarrow \quad d_\mathcal{P}(\widetilde{b}^{ij}, \widetilde{b}^{kl}) \geq d_\mathcal{P}(b^i, b^k) - 2 R , \\
  \end{cgather}holds 
  for all $i,j,k,l$, hence holds also for the choice
  \begin{equation*}
    \min_{\substack{ \forall i \neq k \\ \forall j,l }} d_\mathcal{P}(\widetilde{b}^{ij}, \widetilde{b}^{kl}) \geq \min_{i \neq j} d_\mathcal{P}(b^i, b^j) - 2 R > 0 . \qedhere
  \end{equation*} %
\end{proof}

\begin{corollary}
  Clearly if  $ d_\mathcal{P} ( ϕ (λ; b^i), b^i) \leq R \; \; \forall b^i \in \up{B}$ then $ \max_{ij} d_\mathcal{P}(b^i, \widetilde{b}^{ij}) < R $.
\end{corollary}

\begin{theorem} \label{separability}
  Let $ϕ$ be
  Lipschitz continuous s.t. $\exists L_b$ with $ d_\mathcal{P} ( ϕ (λ; b), ϕ (λ'; b)) \leq L_b d_Λ (λ, λ') $, and let $ ϕ (λ_0; b) = b $ at some $λ_0$.
  Voronoi separability can be guaranteed through a parameter limit $ d_Λ (λ, λ_0) \leq d_{Λ \max} $, so long as $ d_{Λ \max} < \min_{i \neq j} d_\mathcal{P} (b^i, b^j) / 2 \max_i L_{b^i} $.
\end{theorem}

\begin{proof}
  We have $ d_\mathcal{P} ( ϕ (λ; b^i), b^i) \leq L_{b^i} d_Λ (λ, λ_0) \leq \max_i L_{b^i} d_Λ (λ, λ_0) \leq \max_i L_{b^i} d_{Λ \max} $,
  giving $R = \max_i L_{b^i} d_{Λ \max}$.
\end{proof}

\begin{remark}
  With a compact pattern space, $ \min_{i \neq j} d_\mathcal{P}(b^i, b^j) $ cannot be greater than some $ d_{\mathcal{P} \max} $, so it is necessary that $ d_\mathcal{P} ( ϕ (λ; b), b) \leq R < d_{\mathcal{P} \max}/2 $.
\begin{figure}[htbp]
  \centering
  \includegraphics[width=.6\linewidth]{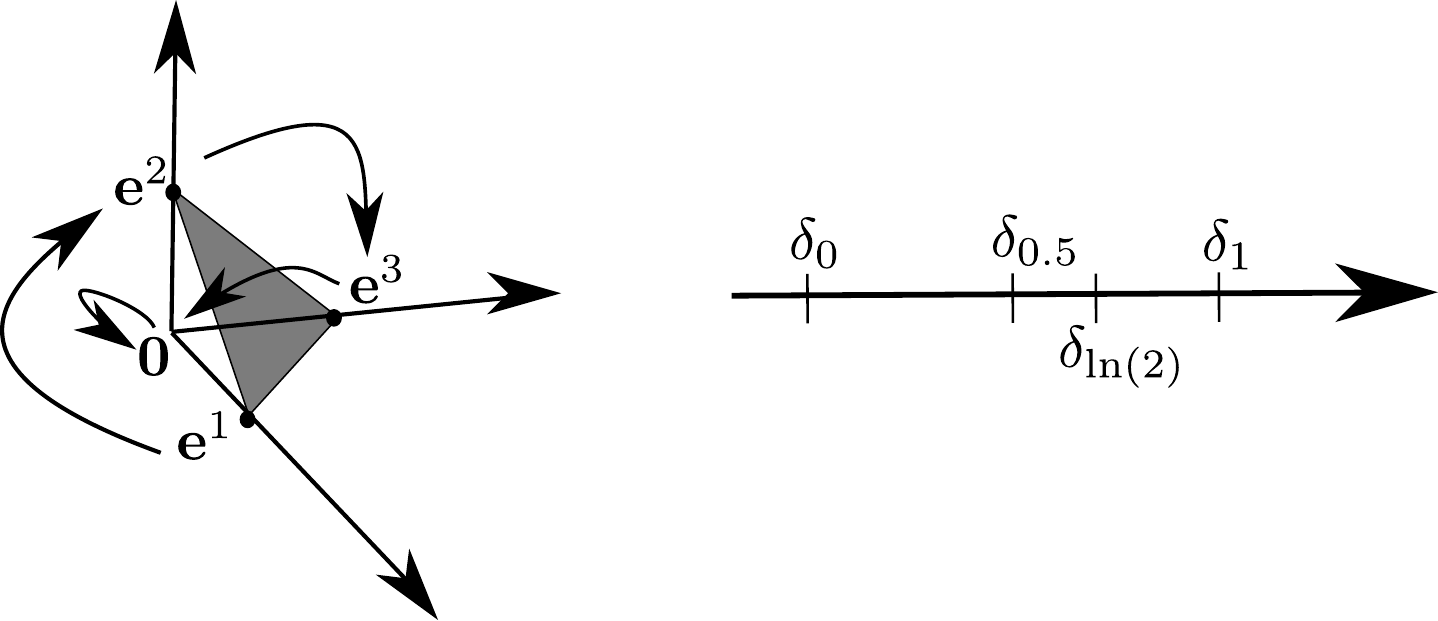}
  \caption{Left: Simplex $\up{S}^2$, where weights on a static grid such as $\mathbf{x}=(0, 0.5, 1)$ can be represented. The repeated application of the shift-down operation (initially on $\mathbf{e}^1$) is illustrated. Right: The Dirac measures $δ_x \in P(\mathbb{R})$ representing distributions with weights concentrated at a single point $x$, which we can continuously interpolate between.} \label{pointwise_nonsmoothness_fig}
\end{figure}
  For instance, with Variations generated by discrete shifting where $ϕ(s; \mathbf{p}) = \mathbf{p}^{\downarrow s} \in \text{Hull} ( \{ \mathbf{e}^i: i \in [K]  \} \cup \{ \mathbf{0} \} )$ (accounting for mass loss from shifting past the end of a vector) with $s \in \{0, 1, 2, \ldots\}$ and $ \mathbf{p} \in S^{N-1} $, we can have $ d_\mathcal{P} ( \mathbf{p}^{\downarrow 1}, \mathbf{p}) = d_{\mathcal{P} \max} \nless d_{\mathcal{P} \max} /2$ (see \cref{pointwise_nonsmoothness_fig}) when $d_\mathcal{P}$ is the Euclidean distance, so basis variation as a result of shifting, is not Voronoi separable over $\mathcal{P} = \text{Hull} ( \{ \mathbf{e}^i: i \in [K]  \} \cup \{ \mathbf{0} \} )$ with the Euclidean distance.
\end{remark}

\section{Wasserstein quantisation and weighted k-means} \label{W2quant}

Minimising
\begin{cgather} \label{eq:quant_cost}
\bar{μ}_N = \argmin_{m \in P(\mathcal{X})} W_2^{d_\mathcal{X}} (μ, m)^2 \\
 \vert \text{supp} (m) \vert \leq N,
\end{cgather}is generally
a semidiscrete \gls{ot} problem\footnote{Where $μ$ could be discrete as well, but only $m$ is constrained to be.}
for the discrete measure $ m = \sum_{i=1}^N \sigw{m}{i} δ_{\sigp[!]{m}{i}} $,
where transport maps between $μ$ and $m$ are necessarily of the form
$T(x) = \sigp[!]{m}{i}$ for $x \in X^i$
where $X^i$ are disjoint so that they partition $\mathcal{X}$.
The weights are determined from mass conservation $ m(\{ \sigp[!]{m}{i} \}) = \sigw{m}{i} = μ( T^{-1} ( \{ \sigp[!]{m}{i} \} ) ) = μ (X^i) $, hence \cref{eq:quant_cost} can be stated as
\begin{equation}
\min_{\substack{\sigp[!]{m}{i} \in \mathcal{X}\\ X^i \subset \mathcal{X} }}
\sum_{i=1}^N \int_{X^i} d_\mathcal{X} \left( x, \sigp[!]{m}{i} \right)^2 d μ (x) .
\end{equation}
This can be solved in an alternating fashion,
where given regions $ \{ X^i \: \forall i \} $,
the corresponding optimal $\sigp[!]{m}{i}$ are centroids\footnote{Where $\sigp[!]{m}{i} = {\int_{X^i} x d μ (x)}/{\int_{X^i} d μ (x)}$ when $d_\mathcal{X}$ is the Euclidean distance} of each region
\begin{equation}
\sigp[!]{m}{i} = \argmin_{x' \in \mathcal{X}} \int_{X^i} d_\mathcal{X} \left( x, x' \right)^2 d μ (x)
\end{equation}
and given $ \text{supp} (m) = \{ \sigp[!]{m}{i} \: \forall i \} $, the optimal regions are then given by the Voronoi tessellation
\begin{equation} \label{voronoi_region}
X^i= \{ x \in \mathcal{X} : d_\mathcal{X} (x, \sigp[!]{m}{i}) < d_\mathcal{X} (x, \sigp[!]{m}{j}) \: \forall j \neq i \} .
\end{equation}
To see this, suppose $X^i$ violated this condition, then for some $j$, $\hat{X} = \{ x \in X^i : d_\mathcal{X} (x, \sigp[!]{m}{i}) > d_\mathcal{X} (x, \sigp[!]{m}{j}) \}$ is non-empty, then one can reassign each $x \in \hat{X}$ according to $\bar{X}^i = X^i \setminus \hat{X}$ and $\bar{X}^j = X^j \cup \hat{X}$ and by doing so improve the cost
\begin{calign}
& \int_{\bar{X}^i} d_\mathcal{X} \left( x, \sigp[!]{m}{i} \right)^2 d μ (x) +
\int_{\hat{X}} d_\mathcal{X} \left( x, \sigp[!]{m}{i} \right)^2 d μ (x) +
\int_{X^j} d_\mathcal{X} \left( x, \sigp[!]{m}{j} \right)^2 d μ (x) \\
> &
\int_{\bar{X}^i} d_\mathcal{X} \left( x, \sigp[!]{m}{i} \right)^2 d μ (x) +
\int_{\hat{X}} d_\mathcal{X} \left( x, \sigp[!]{m}{j} \right)^2 d μ (x) +
\int_{X^j} d_\mathcal{X} \left( x, \sigp[!]{m}{j} \right)^2 d μ (x) ,
\end{calign}hence violating \cref{voronoi_region} is suboptimal.
Regardless of the strict optimality of this step, the sequence of updates produced by alternating these steps are non-increasing and the cost is bound below by $0$, hence this approach converges to a local minimiser of \cref{eq:quant_cost}.

This alternating approach is a generalised, weighted version of the k-means algorithm which can be seen as a generalisation of Lloyds algorithm \cite{Lloyd82}, for which $\mathcal{X}=\mathbb{R}$, $X^i = \{ x : x_\text{mid}^{i-1} < x < x_\text{mid}^{i} \}$ and $ x_\text{mid}^{i} = ( \sigp[!]{m}{i} + \sigp[!]{m}{i+1})/2 $ with $x_\text{mid}^0 = -\infty$ and $x_\text{mid}^N = \infty$.

\subsection{Wasserstein quantiser monotonicity and convergence} \label{W2quantmono}

\begin{theorem}[Monotonicity] \label{ot_quantiser_mono}
Let $μ \in P (\mathcal{X})$, its optimal $W_2$ quantisers of size $N$ ($\bar{μ}_N$) have monotonically decreasing error
\begin{equation}
  W_2^2 (μ, \bar{μ}_{N}) > W_2^2 (μ, \bar{μ}_{N+1}) .
\end{equation}
This follows from being able to construct a suboptimal quantiser of size $N+1$ that is better than the optimal one of size $N$.
\end{theorem}

\begin{proof}
Let $μ \in P (\mathcal{X})$ and its optimal quantiser of size $N$ be $\bar{μ}_N$.
Pick $ x' \in X^N \setminus \{ x^N \}$ and let $A,B$ be the corresponding Voronoi sets of $x^n, x'$ partitioning $X^N$, we may construct
\begin{equation}
  \widetilde{μ}_{N+1} = \bar{μ}_N - μ(X^N) δ_{x^N} + μ(A) δ_{x^N} + μ(B) δ_{x'} ,
\end{equation}
where we have let $ \sigp[!]{\bar{μ}_N}{i}={x}^i$ for brevity.
We have that
\begin{calign}
  \int_{X^N} d_\mathcal{X} \left( x, x^N \right)^2 d μ (x) &= \int_{A} d_\mathcal{X} \left( x, x^N \right)^2 d μ (x) + \int_{B} d_\mathcal{X} \left( x, x^N \right)^2 d μ (x) \\
                                                 &> \int_{A} d_\mathcal{X} \left( x, x^N \right)^2 d μ (x) + \int_{B} d_\mathcal{X} \left( x, x' \right)^2 d μ (x) ,
\end{calign}since 
by definition $B = \{ x \in X^N : d(x, x') < d(x, x^N) \}$.
Hence $ W_2^2 (μ, \bar{μ}_{N}) - W_2^2 (μ, \widetilde{μ}_{N+1} ) >0 $ and so we have monotonicity
\begin{equation*}
  W_2^2 (μ, \bar{μ}_{N}) > W_2^2 (μ, \widetilde{μ}_{N+1} ) \geq  W_2^2 (μ, \bar{μ}_{N+1}) .  \qedhere
\end{equation*}
\end{proof} %

\begin{theorem}[Convergence]
Let $μ \in P (\mathcal{X})$, its optimal $W_2$ quantisers of size $N$ ($\bar{μ}_N$) weakly converge to it.
\begin{equation}
  W_2^2 (μ, \bar{μ}_{N}) \rightarrow 0
  \quad \text{as} \quad
  N \rightarrow \infty
\end{equation}
\end{theorem}

\begin{proof}
  Consider a sequence of suboptimal quantisers $\tilde{μ}_{N}$ that converge to $μ$, then optimal quantisers must also converge.
  Construct such a sequence by noting that at each $N$, one can pick the largest diameter Voronoi region, where $\text{diam}(X) = \max_{x,y \in X} d_\mathcal{X} \left( x, y \right)$, over which one has a 1 Dirac distribution, and by monotonicity we can generate a 2 Dirac quantisation over this region with a better cost.
  Clearly the diameter of all regions shrinks to $0$ in the limit.
\begin{calign}
  W_2^2 (μ, \bar{μ}_{N}) < W_2^2 (μ, \tilde{μ}_{N}) &= \sum_{i=1}^N \int_{X^i_{\tilde{μ}_N}} d_\mathcal{X} \left( x, \sigp[!]{\tilde{μ}_N}{i} \right)^2 d μ (x) \\
                     &\leq \sum_{i=1}^N μ(X^i_{\tilde{μ}_N}) \text{diam}(X^i_{\tilde{μ}_N})^2 \\
                     &\leq \left[  \sum_{i=1}^N μ(X^i_{\tilde{μ}_N}) \right] \max_i \text{diam}(X^i_{\tilde{μ}_N})^2  \rightarrow 0
\end{calign}where
$ \sum_{i=1}^N μ(X^i_{\tilde{μ}_N}) = 1 $.
\end{proof}

\FloatBarrier
\section{Weighted centroids parameterise the geodesic} \label{average_along_geodesics}

\begin{lemma} \label{lem:average_along_geodesics}
    Provided that $ \forall p^1, p^2 \in \mathcal{P}$ and $ \forall d_{1b}, d_{2b} \geq 0 $
    with $ d_{1b} + d_{2b} \geq d(p^1, p^2) = d_{12} $ $ \exists p^b \in \mathcal{P} $
    s.t. $ d(p^1,p^b)=d_{1b} $ and $ d(p^2,p^b)=d_{2b} $,
    then the centroids $ p^b (λ) = B_{d} (p^1, p^2, λ)$ lie on the geodesic between $p^1$ and $p^2$. Further, the curve given by $p^b (λ)$ is constant speed and is therefore the geodesic between $p^1$ and $p^2$.
\end{lemma}

\begin{proof}
    $B(p^1, p^2, λ)$ is attained when the barycentric cost takes the value
    \begin{cgather}
          \min_{ d_{1b}, d_{2b} \in \mathbb{R}_+ } (1-λ) d_{1b}^2 + λ d_{2b}^2 \\
          \text{s.t.} \quad
          d_{1b} + d_{2b} \geq d_{12}
    \end{cgather}with Lagrangian
    \begin{equation}
        \mathcal{L} (d_{1b}, d_{2b}, γ) = 
        λ d_{1b}^2 + (1-λ) d_{2b}^2 + γ (d_{12} - d_{1b} - d_{2b} ) .
    \end{equation}
    It is necessary for optimality that we have stationarity (Karush-Kuhn-Tucker conditions), $ \partial_{d_{1b}} \mathcal{L} = \partial_{d_{2b}} \mathcal{L} = 0$, hence $ γ = 2 (1-λ) d_{1b} = 2 λ d_{2b} $.
    Outside the trivial cases $p^1=p^2$ or $λ\in \{ 0,1 \}$, it follows that $γ > 0$. 
    Since $ γ \neq 0 $,
    the bound is tight, i.e. $ d_{1b} + d_{2b} = d_{12} $ indicating that $\forall λ \in (0,1)$, $ p^b (λ) $ sits on the geodesic connecting $p^1, p^2$. 
    Further, $d_{2b} = d_{1b} (1-λ)/λ $, so that $ d_{12} = \left[ 1 + (1-λ)/λ \right] d_{1b} = d_{1b}/λ $.
    Since we know $p^b (λ)$ at worst reparameterises the geodesic, 
    $d(p^1, p^b(λ)) = λ d(p^1, p^2)$ is enough to know it is constant speed.
\end{proof}

\FloatBarrier
\section{General McCann interpolation} \label{general_mccan_interp}

The following is related to the Monge \gls{ot} minimisation
\begin{equation} \label{eq:dynamic_monge}
  \min_{T^t } \left\{ \int C[ (T^t (x))_{0 \leq t \leq 1} ] d μ (x) :
  \quad T^0 = \text{Id} \quad \text{and} \quad T^1_\# μ = ν \right\} ,
\end{equation}
that is to say, rather than evaluating the final displacement of $x$ to $y$,
evaluate the whole trajectory taken between these. 
This is one dynamic formulation of the \gls{ot} problem.
At time $t$, the mass at $x$ is instead at $T^t (x)$ and only at $t=1$ does it arrive at $y$.
These two problems are equivalent (with $T^{*1} = T^{*}$) provided that the static cost agrees with the cost of the best trajectory, i.e. that
\begin{equation}
  c(x,y) = \min_{(T^t (x) )_{0 \leq t \leq 1}} \left\{ C[(T^t (x))_{0 \leq t \leq 1}]
  : T^0 (x) = x \quad \text{and} \quad T^1 (x) = y \right\}
\end{equation}
is met \cite{Villani2003TopicsIO}.

\begin{definition}[Length of a trajectory] \label{def:length}
   
  For the trajectory $τ: [0,1] \rightarrow \mathcal{X}$, the length is defined as
  \begin{equation}
    L_{d_\mathcal{X}} [τ] = \sup \sum_{i=1}^N d_{\mathcal{X}} (τ(t_{i-1}), τ(t_i)) ,
  \end{equation}
  with the supremum taken over all $N$ and sequences $0 = t_0 \leq \ldots \leq t_N = 1$.
  
\end{definition}

\begin{definition}[Length space]
  A metric space $(\mathcal{X}, d_\mathcal{X})$ with $d_\mathcal{X} (x,y) = \inf_{τ} L[τ]$ over all $τ: [0,1] \rightarrow \mathcal{X}$ with $τ(0)=x$ and $τ(1)=y$.
\end{definition}

\begin{definition}[Geodesic]
  A geodesic from $ x \in \mathcal{X}$ to $ y \in \mathcal{X}$, is a trajectory $τ: [0,1] \rightarrow \mathcal{X}$ with $τ(0)=x$, $τ(1)=y$ which has (1.) minimal length and (2.) constant speed.
  (1.) can be ensured by
  $L[(τ)_{t^1 \leq t \leq t^2}] = d_\mathcal{X} (τ(t^1), τ(t^2))$
  for all $0 \leq t^1 < t^2 \leq 1$,
  and (2.) by $ L[(τ)_{t^1 \leq t \leq t^2}] = (t^2 - t^1) λ $ for some $λ$,
  again for all $0 \leq t^1 < t^2 \leq 1$.
  By \cref{def:length}, 
  (1.) is given by (2.) when $λ = d_\mathcal{X} (x, y)$.
\end{definition}

\begin{definition}[Geodesic space]
  A metric space $(\mathcal{X}, d_\mathcal{X})$ in which $\forall x,y \in \mathcal{X}$, there exists a geodesic connecting $x$ and $y$.
  Geodesic spaces are also length spaces.
\end{definition}

Suppose $c = d_\mathcal{X}^p$ where the ground space $(\mathcal{X}, d_\mathcal{X})$ is a geodesic space,
if we let $C [τ] = L_{d_\mathcal{X}} [τ]^p$
and without loss of correspondence to the Monge problem, 
let \cref{eq:dynamic_monge} also require that all considered paths are constant speed, 
then clearly the optimal trajectories $ (T^{*t} (x))_{0 \leq t \leq 1} $ are ground space geodesics between $x$ and $T^{*} (x)$.
By \cref{lem:average_along_geodesics}, this optimal trajectory can then be expressed in terms of ground space barycenters $ T^{*t} (x) = B_{d_\mathcal{X}} (x, T^* (x), t) $, where $T^*$ is the optimal map in the static problem.

\begin{lemma} \label{lem:dyn_map_geodesics}
  Let the \gls{ot} problem between $μ$ and $ν$ admit a Monge solution and
  let $T^{*t}$ be the (constant speed) dynamic \gls{ot} solution.
  Let $T^{12}$ be the (static) optimal transport map between $T^{*t^1}_\# μ$ and $T^{*t^2}_\# μ$.
  Provided that $T^{12} (T^{*t^1} (x)) = T^{*t^2} (x)$ for all $0 \leq t^1 < t^2 \leq 1$, 
  then $( T^{*t}_\# μ )_{0 \leq t \leq 1}$ is a Wasserstein space geodesic between $μ$ and $ν$.
\end{lemma}

\begin{proof}
  Note that the shortest path between $T^{*t^1} (x)$ to $ T^{*t^2} (x)$ is given by the geodesic segments $(T^{*t} (x))_{t^1 \leq t \leq t^2}$, then
\begin{align*}
  W_p^{d_\mathcal{X}} (T^{*t^1}_\# μ, T^{*t^2}_\# μ )^p 
  &= \int d_\mathcal{X} (x, T^{12} (x))^p d (T^{*t^1}_\# μ) (x) \\
  &= \int d_\mathcal{X} (T^{*t^1} (x), T^{*t^2} (x) )^p d μ (x) \\
  &= \int L_{d_\mathcal{X}} [ (T^{*t} (x))_{t^1 \leq t \leq t^2} ]^p d μ (x) \\                                                      
  &= \int [(t^2-t^1) d_\mathcal{X} (x, T^{*} (x))]^p d μ (x) = (t^2-t^1)^p W_p^{d_\mathcal{X}} (μ, ν)^p . \qedhere
\end{align*}
\end{proof}

Where \cref{lem:dyn_map_geodesics} applies, such as in 1D, where feasible monotonic plans are unique,
\cref{lem:average_along_geodesics} can be employed again, allowing one to calculate Wasserstein barycenters via $B (μ, ν, λ) = T^{*λ}_\# μ$, this is the displacement interpolation.
This is usually stated for an Euclidean ground space where $T^{*λ} = (1-λ) \text{Id} + λ T^*$.

\end{appendices}

\FloatBarrier
\bibliographystyle{plainnat}
\bibliography{references_CS,references_chem,references_HU,references_other,references_OT}

\end{document}